\documentclass[runningheads,a4paper]{llncs}

\usepackage{amssymb}
\usepackage{graphicx}
\usepackage{epstopdf}
\usepackage{url}
\usepackage{subfigure}
\urldef{\mailsa}\path|firstname.lastname@liu.se|

\newtheorem{mydef}{Definition}

\newtheorem{mytheorem}{Theorem}
\newtheorem{mylemma}{Lemma}

\begin{document}

\mainmatter  

\title{Get my pizza right: \\
Repairing missing is-a relations \\
in ${\cal ALC}$ ontologies \\
(extended version)}

\titlerunning{Repairing missing is-a relations in ${\cal ALC}$ ontologies}
%
%
\author{Patrick Lambrix
\and  Zlatan Dragisic
\and Valentina Ivanova}

\institute{Department of Computer and Information Science\\
and Swedish e-Science Research Centre\\
Link\"oping University, 581 83 Link\"oping, Sweden\\
\mailsa\\
\url{http://www.ida.liu.se/}
}

%
%

\maketitle

\begin{abstract}
With the increased use of ontologies in semantically-enabled applications, the issue of debugging defects in ontologies has become increasingly important. These defects can lead to wrong or incomplete results for the applications.
Debugging consists of the phases of detection and repairing. In this paper we focus on the repairing phase of a particular kind of defects, i.e. the missing relations in the is-a hierarchy. Previous work has dealt with the case of taxonomies. In this work we extend the scope to deal with ${\cal ALC}$ ontologies that can be represented using acyclic terminologies. We present algorithms and discuss a system.
\end{abstract}

{\bf This is an extended version of \cite{LDI12}.}

\section{Introduction}
\label{section-introduction}

Developing ontologies is not an easy task, and often the resulting ontologies are not consistent or complete. Such ontologies, although often useful, also lead to problems when used in semantically-enabled applications. Wrong conclusions may be derived or valid conclusions may be missed.  Defects in ontologies can take different forms (e.g. \cite{KPSH06}). Syntactic defects are usually easy to find and to resolve. Defects regarding style include such things as unintended redundancy. More interesting and severe defects are the modeling defects which require domain knowledge to detect and resolve, and semantic defects such as unsatisfiable concepts and inconsistent ontologies. Most work up to date has focused on debugging (i.e. detecting and repairing) the semantic defects in an ontology (e.g. \cite{KPSH06,KPSC06,Sch05,CRVP09}). Modeling defects have mainly been discussed in \cite{BH08,LLT09,LL11} for taxonomies, i.e. from a knowledge representation point of view, a simple kind of ontologies. The focus has been on defects regarding the is-a structure of the taxonomies.  In this paper we tackle the problem of repairing the is-a structure of ${\cal ALC}$ ontologies that can be represented using acyclic terminologies. 

In addition to its importance for the correct modeling of a domain, the structural information in ontologies is also important in semantically-enabled applications. For instance, the is-a structure is used in ontology-based search and annotation. In ontology-based search, queries are refined and expanded by moving up and down the hierarchy of concepts. Incomplete structure in ontologies influences the quality of the search results. As an example, suppose we want to find articles in the MeSH (Medical Subject Headings \cite{MeSH}, controlled vocabulary of the National Library of Medicine, US) Database of PubMed \cite{PubMed} using the term \emph{Scleral Diseases} in MeSH. By default the query will follow the hierarchy of MeSH and include more specific terms for searching, such as \emph{Scleritis}. If the relation between \emph{Scleral Diseases} and \emph{Scleritis} is missing in MeSH, we will miss 738 articles in the search result, which is about 55\% of the original result. 
The structural information is also important information in ontology engineering research. For instance, most current ontology alignment systems use structure-based strategies to find mappings between the terms in different ontologies (e.g. overview in \cite{LST09}) and the modeling defects in the structure of the ontologies have an important influence on the quality of the ontology alignment results. 

Debugging modeling defects in ontologies consists of two phases, detection and repair. There are different ways to detect missing is-a relations. One way is by inspection of the ontologies by domain experts. 
Another way is to use external knowledge sources. For instance, there is much work on finding relationships between terms in the ontology learning area \cite{CBM05}. In this setting, new ontology elements are derived from text using knowledge acquisition techniques. Regarding the detection of is-a relations, one paradigm is based on linguistics using lexico-syntactic patterns. The pioneering research conducted in this line is in \cite{Hea92}, which defines a set of patterns indicating is-a relationships between words in the text.  Another paradigm is based on machine learning and statistical methods. Further, guidelines based on logical patterns can be used \cite{CRVP09}.
When the ontology is part of a network of ontologies connected by mappings between the ontologies, knowledge inherent in the ontology network can be used to detect missing is-a relations using logical derivation \cite{BH08,LLT09,LL11}. 
However, although there are many approaches to detect missing is-a relations, these approaches, in general, do not detect {\it all} missing is-a relations. For instance, although the precision for the linguistics-based approaches is high, their recall is usually very low.

In this paper we assume that the detection phase has been performed. We assume that we have obtained a set of missing is-a relations for a given ontology and focus on the repairing phase. In the case where our set of missing is-a relations contains {\it all} missing is-a relations, the repairing phase is easy. We just add all missing is-a relations to the ontology and a reasoner can compute all logical consequences. However, when the set of missing is-a relations does not contain all missing is-a relations - and this is the common case - there are different ways to repair the ontology.
The easiest way is still to just add the missing is-a relations to the ontology. For instance, Figure \ref{small-pizza-ontology} shows a small part of a pizza ontology based on \cite{pizza}, that is relevant for our discussions. Assume that we have detected that  MyPizza $\dot{\sqsubseteq}$  FishyMeatyPizza and MyFruttiDiMare $\dot{\sqsubseteq}$ NonVegetarianPizza are missing is-a relations. Obviously, adding MyPizza $\dot{\sqsubseteq}$  FishyMeatyPizza and MyFruttiDiMare $\dot{\sqsubseteq}$ NonVegetarianPizza to the ontology will repair the missing is-a structure.
 However, there are other more interesting possibilities. For instance, adding AnchoviesTopping $\dot{\sqsubseteq}$ FishTopping and ParmaHamTopping $\dot{\sqsubseteq}$ MeatTopping will also repair the missing is-a structure. Another more informative way\footnote{This is more informative in the sense that the former is derivable from the latter. Adding ParmaHamTopping $\dot{\sqsubseteq}$ HamTopping, also allows to derive  ParmaHamTopping $\dot{\sqsubseteq}$ MeatTopping as the ontology already includes HamTopping $\dot{\sqsubseteq}$ MeatTopping.} to repair the missing is-a structure is to add AnchoviesTopping $\dot{\sqsubseteq}$ FishTopping and ParmaHamTopping $\dot{\sqsubseteq}$ HamTopping. Essentially, these other possibilities to repair the ontology include missing is-a relations (e.g. AnchoviesTopping $\dot{\sqsubseteq}$ FishTopping) that were not originally detected by the detection algorithm.\footnote{Therefore, the approach discussed in this paper can also be seen as a detection method that takes already found missing is-a relations as input.}
We also note that from a logical point of view, adding
AnchoviesTopping $\dot{\sqsubseteq}$ MeatTopping and ParmaHamTopping $\dot{\sqsubseteq}$ FishTopping also repairs the missing is-a structure. However, from the point of view of the domain, this solution is not correct. Therefore, as it is the case for all approaches for debugging modeling defects, a domain expert needs to validate the logical solutions.

The contributions of this paper are threefold. First, we show that the problem of finding possible ways to repair the missing is-a structure in an ontology in general can be formalized as a generalized version of the TBox abduction problem (Section \ref{section-problem}). Second, we propose an algorithm to generate different ways to repair the missing is-a structure in ${\cal ALC}$ ontologies that can be represented using acyclic terminologies (Section \ref{section-algorithm}). Third, we discuss a system that allows a domain expert to repair the missing is-a structure in Section  
\ref{section-implementation}. We discuss the functionality and user interface of the repairing system and show an example run. Further, we discuss related work in Section \ref{section-related-work} and conclude in Section \ref{section-conclusion}. We continue, however, with some necessary preliminaries in Section \ref{section-preliminaries}.

\begin{figure}[tb]
\scriptsize
\begin{center}
\begin{tabular}{| l |} \hline

Pizza  $\dot{\sqsubseteq}$ $\top$ \\
PizzaTopping  $\dot{\sqsubseteq}$ $\top$\\
hasTopping  $\dot{\sqsubseteq}$ $\top$ $\times$ $\top$\\

AnchoviesTopping  $\dot{\sqsubseteq}$ PizzaTopping \\

MeatTopping $\dot{\sqsubseteq}$ PizzaTopping\\

HamTopping $\dot{\sqsubseteq}$ MeatTopping\\

ParmaHamTopping $\dot{\sqsubseteq}$ PizzaTopping \\

FishTopping $\dot{\sqsubseteq}$ PizzaTopping $\sqcap$ $\neg$MeatTopping \\

TomatoTopping $\dot{\sqsubseteq}$ PizzaTopping $\sqcap$  $\neg$MeatTopping $\sqcap$  $\neg$FishTopping\\

GarlicTopping $\dot{\sqsubseteq}$ PizzaTopping $\sqcap$  $\neg$MeatTopping $\sqcap$  $\neg$FishTopping\\

MyPizza $\doteq$  Pizza $\sqcap$ $\exists$ hasTopping.AnchoviesTopping $\sqcap$ $\exists$ hasTopping.ParmaHamTopping \\

FishyMeatyPizza  $\doteq$  Pizza $\sqcap$ $\exists$ hasTopping.FishTopping $\sqcap$ $\exists$ hasTopping.MeatTopping \\

MyFruttiDiMare $\doteq$ Pizza $\sqcap$ $\exists$ hasTopping.AnchoviesTopping \\
\hspace{1.5cm} $\sqcap$  $\exists$ hasTopping.GarlicTopping 
 $\sqcap$ $\exists$ hasTopping.TomatoTopping \\
\hspace{1.5cm} $\sqcap$ $\forall$ hasTopping.(AnchoviesTopping $\sqcup$ GarlicTopping  $\sqcup$ TomatoTopping)\\

VegetarianPizza  $\doteq$ Pizza $\sqcap$ $\neg$ $\exists$ hasTopping.FishTopping $\sqcap$ $\neg$ $\exists$ hasTopping.MeatTopping\\

NonVegetarianPizza $\doteq$ Pizza $\sqcap$  $\neg$VegetarianPizza\\

\hline
\end{tabular}
\end{center}
\caption{A pizza ontology.}
\label{small-pizza-ontology}
\end{figure}

\begin{figure}
\scriptsize
\begin{center}
\begin{tabular}{| l |} \hline

AnchoviesTopping  $\doteq$ PizzaTopping  $\sqcap$ $\overline{AnchoviesTopping}$    \\

MeatTopping $\doteq$ PizzaTopping   $\sqcap$  $\overline{MeatTopping}$ \\

HamTopping $\doteq$ MeatTopping  $\sqcap$  $\overline{HamTopping}$ \\

ParmaHamTopping $\doteq$ PizzaTopping  $\sqcap$   $\overline{ParmaHamTopping}$ \\

FishTopping $\doteq$ PizzaTopping $\sqcap$ $\neg$MeatTopping  $\sqcap$   $\overline{FishTopping}$  \\

TomatoTopping $\doteq$ PizzaTopping $\sqcap$  $\neg$MeatTopping $\sqcap$  $\neg$FishTopping  $\sqcap$   $\overline{TomatoTopping}$  \\

GarlicTopping $\doteq$ PizzaTopping $\sqcap$  $\neg$MeatTopping $\sqcap$  $\neg$FishTopping $\sqcap$   $\overline{GarlicTopping}$   \\

MyPizza $\doteq$  Pizza $\sqcap$ $\exists$ hasTopping.AnchoviesTopping $\sqcap$ $\exists$ hasTopping.ParmaHamTopping \\

FishyMeatyPizza  $\doteq$  Pizza $\sqcap$ $\exists$ hasTopping.FishTopping $\sqcap$ $\exists$ hasTopping.MeatTopping \\

MyFruttiDiMare $\doteq$ Pizza $\sqcap$ $\exists$ hasTopping.AnchoviesTopping \\
\hspace{1.5cm} $\sqcap$  $\exists$ hasTopping.GarlicTopping 
 $\sqcap$ $\exists$ hasTopping.TomatoTopping \\
\hspace{1.5cm} $\sqcap$ $\forall$ hasTopping.(AnchoviesTopping $\sqcup$ GarlicTopping  $\sqcup$ TomatoTopping)\\

VegetarianPizza  $\doteq$ Pizza $\sqcap$ $\neg$ $\exists$ hasTopping.FishTopping $\sqcap$ $\neg$ $\exists$ hasTopping.MeatTopping\\

NonVegetarianPizza $\doteq$ Pizza $\sqcap$  $\neg$VegetarianPizza\\

\hline
\end{tabular}
\end{center}
\caption{A pizza ontology - Acyclic ${\cal ALC}$ terminology.}
\label{small-pizza-ontology-acyclic-terminology}
\end{figure}

\section{Preliminaries}
\label{section-preliminaries}

In this paper we deal with ontologies represented in the description logic ${\cal ALC}$ with acyclic terminologies (e.g. \cite{BS01}). Concept descriptions are defined using constructors as well as concept and role names. As constructors ${\cal ALC}$ allows concept conjunction (C $\sqcap$ D), disjunction (C $\sqcup$ D), negation ($\neg$ C), universal quantification ($\forall$ r.C) and existential quantification ($\exists$ r.C).\footnote{C and D represent concepts, and r represents a role.} In this paper we consider ontologies that can be represented by a TBox that is an acyclic terminology. An acyclic terminology is a finite set of concept definitions (i.e. terminological axioms of the form $C$ $\doteq$ $D$ where $C$ is a concept name) that neither contains multiple definitions nor cyclic definitions.\footnote{We observe that the TBox in Figure \ref{small-pizza-ontology} is not an acyclic terminology as there are statements of the form $A$ $\dot{\sqsubseteq}$ $C$. However, it is possible to create an equivalent TBox that is a acyclic terminology by replacing the statements of the form $A$ $\dot{\sqsubseteq}$ $C$ with $A$ $\doteq$ $C$ $\sqcap$  $\overline{A}$ where $\overline{A}$ is new atomic concept. See Figure \ref{small-pizza-ontology-acyclic-terminology}.}
 An ABox contains assertional knowledge, i.e. statements about the membership of individuals (interpreted as elements in the domain) to concepts (C(i)) as well as relations between individuals (r(i,j)).\footnote{i and j represent individuals. In the completion graph in Section \ref{section-algorithm} statements of the form C(i) are written as $i:C$, and statements of the form r(i,j) are written as $irj$.} A knowledge base contains a TBox and an ABox. A model of the TBox/ABox/knowledge base satisfies all axioms of the TBox/ABox/knowledge base. A knowledge base is consistent if it does not contain contradictions.

An important reasoning service is the checking of (un)satisfilibility of concepts (a concept is unsatisfiable if it is necessarily interpreted as the empty set in all models of the TBox/ABox/knowledge base, satisfiable otherwise). A TBox is incoherent if it contains an unsatisfiable concept.

Checking satisfiability of concepts in ${\cal ALC}$ can be done using a tableau-based algorithm (e.g. \cite{BS01}). To test whether a concept C is satisfiable such an algorithm starts with an ABox containing the statement C(x) where x is a new individual and it is usually assumed that C is normalized to negation normal form. It then applies consistency-preserving transformation rules to the ABox (Figure \ref{transformation-rules}). The $\sqcap$-, $\forall$- and $\exists$-rules extend the ABox while the $\sqcup$-rule creates multiple ABoxes. The algorithm continues applying these transformation rules to the ABoxes until no more rules apply. This process is called completion and if one of the final ABoxes does not contain a contradiction (we say that it is open), then satisfiability is proven, otherwise unsatisfiability is proven.
One way of implementing this approach is through completion graphs which are directed graphs in which every node represents an ABox.  Application of the $\sqcup$-rule produces new nodes with one statement each, while the other rules add statements to the node on which the rule is applied. The ABox for a node contains all the statements of the node as well as the statements of the nodes on the path to the root. Satisfiability is proven if at least one of the ABoxes connected to a leaf node does not contain a contradiction, otherwise unsatisfiability is proven.

In this paper we assume that an ontology O is represented by a knowledge base containing a TBox that is an acyclic terminology and an empty ABox. In this case reasoning can be reduced to reasoning without the TBox by unfolding the definitions. However, for efficiency reasons, instead of running the previously described satisfiability checking algorithm on an unfolded concept description, the unfolding is usually performed on demand within the satisifiability checking algorithm. It has been proven that satisfiability checking w.r.t. acyclic terminologies is PSPACE-complete in  ${\cal ALC}$ \cite{Lutz99}.

\begin{figure}[tb]
\begin{center}
\scriptsize
\begin{tabular}{| l |}
\hline
$\sqcap$-rule: if the ABox contains (C$_1$ $\sqcap$ C$_2$)(x), but it does not contain both C$_1$(x) and C$_2$(x), \\
\hspace*{1cm} then these are added to the ABox. 
\\
$\sqcup$-rule: if the ABox contains (C$_1$ $\sqcup$ C$_2$)(x), but it contains neither C$_1$(x) nor C$_2$(x), \\
\hspace*{1cm} then two ABoxes are created representing the two choices of adding C$_1$(x) or adding C$_2$(x).
\\
$\forall$-rule: if the ABox contains ($\forall$ r.C)(x) and r(x,y), but it does not contain C(y),\\
\hspace*{1cm} then this is added to the ABox.
\\
$\exists$-rule: if the ABox contains ($\exists$ r.C)(x) but there is no individual z such that r(x,z) and C(z) are in the ABox, \\
\hspace*{1cm} then r(x,y) and C(y) with y an individual name not occurring in the ABox, are added.\\
\hline
\end{tabular}
\end{center}
\caption{Transformation rules (e.g. \protect\cite{BS01}).}
\label{transformation-rules}
\end{figure}

\section{An abduction problem}
\label{section-problem}

In our setting, a missing is-a relation in an ontology O represented by a knowledge base KB, is an is-a relation between named concepts that is not derivable from the KB, but that is correct according to the intended domain. We assume that we have a set M of missing is-a relations (but not necessarily all) for O.
Then, the is-a structure of O can be repaired by adding is-a relations (or axioms of the form C $\dot{\sqsubseteq}$ D) between named concepts to O such that the missing is-a relations can be derived from the repaired ontology. This repair problem can be formulated as a generalized version of the TBox abduction problem. 

\begin{mydef}

Let KB be a knowledge base in ${\cal L}$,
and for 1 $\leq$ i $\leq$ m: C$_i$, D$_i$ are concepts that are satisfiable w.r.t. KB, such that KB $\cup$  \{ C$_i$ $\dot{\sqsubseteq}$ D$_i$ $\mid$ 1 $\leq$ i $\leq$ m\} is coherent. A solution to the generalized TBox abduction problem for (KB, \{(C$_i$, D$_i$) $\mid$ 1 $\leq$ i $\leq$ m\}) is any finite set S$_{GT}$ = \{G$_j$ $\dot{\sqsubseteq}$  H$_j$ $\mid$ j $\leq$ n\} of TBox assertions in ${\cal L}^\prime$ such that $\forall$ i: KB $\cup$  S$_{GT}$ $\models$ C$_i$ $\dot{\sqsubseteq}$ D$_i$. 
The set of all such solutions is denoted by  S$_{GT}$(KB, \{(C$_i$, D$_i$ )$\mid$ 1 $\leq$ i $\leq$ m\}).

\end{mydef}

In our setting the language ${\cal L}$ is ${\cal ALC}$ and ${\cal L}^\prime$ only allows {\it named} concepts. We say that any solution in S$_{GT}$(KB, \{(C$_i$, D$_i$)\}$_i$) is a {\it repairing action}. A repairing action is thus a set of is-a relations.
When {\it m} = 1, this definition of the generalized TBox abduction problem coincides with the TBox abduction problem as formulated in \cite{elsenbroich2006case}, which therefore deals with repairing one missing is-a relation. 
Further, we have that  
 S$_{GT}$(KB, \{(C$_i$, D$_i$)\}$_i$) = $\cap_i$ S$_{GT}$(KB, \{(C$_i$, D$_i$)\}).\footnote{A solution for all missing is-a relations is also a solution for each missing is-a relation and therefore, in the intersection of the solutions concerning one missing is-a relation at the time. Further, a solution in the intersection of the sets of solutions for each of the missing is-a relations, allows, when added to the knowledge base, to derive all missing is-a relations, and is therefore, a solution for the generalized TBox abduction problem as well.} \\
This shows that solving a generalized TBox abduction problem can be done by solving {\it m} TBox abduction problems and then taking the intersection of the solutions. In practice, however, this leads to a number of difficulties. First, it would mean that a domain expert will need to choose between large sets of repairing actions for all the missing is-a relations at once, and this may be a very hard task. Further, due to the size of the solution space, even generating all solutions for one TBox abduction problem is, in general, infeasible. Also, many of the solutions will not be interesting for a domain expert (e.g. \cite{LLT09}). For instance, there may be solutions containing is-a relations that do not contribute to the actual repairing. Some solutions may introduce unintended equivalence relations.

A common way to limit the number of possible solutions is to introduce constraints, e.g. minimality. Our proposed algorithm generates solutions to a TBox abduction problem that are minimal in the sense that repairing actions only contain necessary information for repairing the missing is-a relations. 
Further, we check the generated solutions for the introduction of incoherence.

\section{Algorithm for generating repairing actions}
\label{section-algorithm}

\begin{figure}[tb]
\begin{center}
\scriptsize
\begin{tabular}{| l |}
\hline
\noindent\textbf{\textit{Input}}:
The ontology O represented by knowledge base KB and a set of missing is-a relations M.\\
\textit{\textbf{Output}}: Set of repairing actions $Rep(M)$.\\
\textbf{\textit{Algorithm}} \\
1. For every missing is-a relation $A_i$ $\dot{\sqsubseteq}$ $B_i$ in M:\\
\hspace*{0.35cm} 1.1 $G = $ completion graph after running a tableaux algorithm with $A_i \sqcap \lnot B_i$ as input on KB;\\
\hspace*{0.35cm} 1.2 {\it Leaf-ABoxes} =  get ABoxes of the leaves from the completion graph $G$; \\ 
\hspace*{0.35cm} 1.3 For every open ABox $\mathcal{A} \in$ {\it Leaf-ABoxes}: \\
\hspace*{0.7cm}  1.3.1 $R_{\mathcal{A}} = \emptyset $; \\
\hspace*{0.7cm}  1.3.2 For every individual $x_j$ in $\mathcal{A}$; \\
\hspace*{1.05cm} 1.3.2.1 $Pos_{x_j} = \{ P \mid x_j:P \in \mathcal{A} \wedge P$ is a named concept\};
\\
\hspace*{1.05cm} 1.3.2.2 $Neg_{x_j} = \{ N \mid x_j:\neg N \in \mathcal{A} \wedge N$ is a named concept\}; \\
\hspace*{1.05cm} 1.3.2.3 $R_{\mathcal{A}}$ =   $R_{\mathcal{A}} \cup \{ P$ $\dot{\sqsubseteq}$ $ N \mid  P \in Pos_{x_j} \land N \in Neg_{x_j}\}$; \\
\hspace*{0.35cm} 1.4 $Rep(A_i, B_i) = \emptyset $; \\
\hspace*{0.35cm} 1.5 As long as there are different choices: \\
\hspace*{0.7cm}  1.5.1 Create a repairing action $ra$ by choosing one element from each set $R_{\mathcal{A}}$; \\
\hspace*{0.7cm}  1.5.2 $Rep(A_i, B_i)$ = $Rep(A_i, B_i)$ $\cup$ \{ $ra$ \}; \\
\hspace*{0.7cm}  1.5.3 Remove reduncancy in $Rep(A_i, B_i)$; \\
\hspace*{0.7cm}  1.5.4 Remove incoherent solutions from $Rep(A_i, B_i)$; \\
2.  $Rep(M)$ = \{ M \}; \\
3. As long as there are different choices: \\
\hspace*{0.35cm} 3.1 Create a repairing action $rp$ by choosing one element from each $Rep(A_i, B_i)$ \\
\hspace*{0.75cm} and taking the union of these elements; \\
\hspace*{0.35cm} 3.2 $Rep(M)$ = $Rep(M)$ $\cup$ \{ rp \}; \\
4. Remove reduncancy in $Rep(M)$; \\
5. Remove incoherent solutions from $Rep(M)$; \\
\hline
\end{tabular}
\end{center}
\caption{Basic algorithm for generating repairing actions using completion graph.}
\label{basic-algorithm}
\end{figure}

\subsection{Basic algorithm}

The basic algorithm for generating repairing actions for a set of given missing is-a relations for an ontology is shown in Figure \ref{basic-algorithm}. In this first study we assume that the existing structure in the ontology is correct. Also, as stated in the definition in Section \ref{section-problem}, adding the missing is-a relations to the ontology does not lead to incoherence. 

In step 1 a set of repairing actions is generated for each missing is-a relation (thereby solving a TBox abduction problem for each missing is-a relation). 
In step 1.1 we run the satisfiability checking algorithm with unfolding on demand as described in Section \ref{section-preliminaries}, on KB with input $A_i \sqcap \lnot B_i$, and we collect the ABoxes of the leaves in step 1.2.
As $A_i$ $\dot{\sqsubseteq}$ $B_i$ is a missing is-a relation, it cannot be derived from KB  and thus the completion graph will have open leaf ABoxes. We then generate different ways to close these ABoxes in step 1.3. For each individual  $x$ in an open leaf ABox we collect the concepts in the statements of the form $x$:$A$ in $Pos_x$ and the concepts in the statements of the form $x$:$\lnot B$ in $Neg_x$ where $A$ and $B$ are named concepts. The ABox can be closed if $A$ $\dot{\sqsubseteq}$ $B$ is added
 to the ontology for any  $A$ $\in$ $Pos_x$ and  $B$ $\in$ $Neg_x$. Indeed, with this extra information $x$:$A$ could be unfolded and $x$:$B$ would be added to the ABox, and this gives a contradiction with  $x$:$\neg B$ which was already in the ABox. A repairing action for the missing is-a relation is then a set of $A$ $\dot{\sqsubseteq}$ $B$ that closes each open leaf ABox. In step 1.5 we generate such sets by selecting one such axiom per open leaf ABox and remove redundancy based on the sub-set relation. If a repairing action is a super-set of another repairing action, it is removed. Further, we remove solutions that introduce incoherence.

In step 2 the repairing actions set for all missing is-a relations is initialized with the set of missing is-a relations. Therefore, there will always be at least one repairing action. In step 3 additional repairing actions are generated by combining repairing actions for the individual missing is-a relations. As repairing actions are sets, there are no duplicates in a repairing action. In step 4 we remove redundancy based on the sub-set relation. In step 5 we remove solutions that introduce incoherence. We note that there always will be at least one solution that does not introduce incoherence (i.e. {\it M} or a sub-set of {\it M}).

As an example, consider the acyclic terminology in Figure \ref{small-pizza-ontology-acyclic-terminology}  equivalent to the ontology in Figure \ref{small-pizza-ontology} and M = \{MyPizza $\dot{\sqsubseteq}$  Fishy\-MeatyPizza, MyFruttiDiMare $\dot{\sqsubseteq}$  NonVegetarianPizza\}. For the missing is-a relation MyPizza $\dot{\sqsubseteq}$  Fishy\-MeatyPizza the completion graph obtained after running the satisfiability check on $MyPizza \sqcap \lnot FishyMeatyPizza$ is shown in Figure \ref{fig:cgraph}.
The completion graph contains 17 nodes of which 11 are leaf nodes. Of these leaf nodes 6 are closed and the repairing actions will be based on the 5 open leaf nodes. The computation of the $Pos_{x_j}$, $Neg_{x_j}$ and $R_{\mathcal{A}}$ for the leaf ABoxes is given in Figure \ref{rep-basic-algorithm}. Closing all open leaf ABoxes will lead to 11 non-redundant repairing actions (see Figure \ref{non-redundant-repairing-actions-single}) of which 8 lead to incoherence (marked by * in  Figure \ref{non-redundant-repairing-actions-single}). The remaining repairing actions are \{MyPizza $\dot{\sqsubseteq}$ FishyMeaty\-Pizza\}, \{AnchoviesTopping $\dot{\sqsubseteq}$ FishTopping,  ParmaHamTopping $\dot{\sqsubseteq}$ MeatTopping \} and \{ParmaHamTopping $\dot{\sqsubseteq}$ FishTopping, AnchoviesTopping $\dot{\sqsubseteq}$ MeatTopping\}. 
Observe that in this example we removed the repairing actions containing a concept of the form  $\overline{A}$. This is because whenever $x$:$\overline{A}$ occurs in an ABox, then also $x$:$A$ is in that ABox, as  $x$:$\overline{A}$ only can appear by unfolding $A$ in $x$:$A$. Therefore, whenever $\overline{A}$ is a choice for repairing, then also $A$ is a choice for repairing, and we can ignore the $\overline{A}$ choice, which does not relate to a concept in the original ontology.

\begin{figure}[tb]
\begin{center}
\scriptsize
\begin{tabular}{| l |}
\hline
\{MyPizza $\dot{\sqsubseteq}$ FishyMeatyPizza\} \\
\{Pizza $\dot{\sqsubseteq}$ FishyMeatyPizza\}$^*$\\ 
\{AnchoviesTopping $\dot{\sqsubseteq}$ FishTopping, AnchoviesTopping $\dot{\sqsubseteq}$ MeatTopping\}$^*$ \\
\{PizzaTopping $\dot{\sqsubseteq}$ FishTopping, AnchoviesTopping $\dot{\sqsubseteq}$ MeatTopping\}$^*$ \\
\{ParmaHamTopping $\dot{\sqsubseteq}$ FishTopping, AnchoviesTopping $\dot{\sqsubseteq}$ MeatTopping\} \\
\{AnchoviesTopping $\dot{\sqsubseteq}$ FishTopping, PizzaTopping $\dot{\sqsubseteq}$ MeatTopping\}$^*$\\
\{PizzaTopping $\dot{\sqsubseteq}$ FishTopping, PizzaTopping $\dot{\sqsubseteq}$ MeatTopping\}$^*$\\
\{ParmaHamTopping $\dot{\sqsubseteq}$ FishTopping, PizzaTopping $\dot{\sqsubseteq}$ MeatTopping\}$^*$\\
\{AnchoviesTopping $\dot{\sqsubseteq}$ FishTopping, ParmaHamTopping $\dot{\sqsubseteq}$ MeatTopping\}\\
\{PizzaTopping $\dot{\sqsubseteq}$ FishTopping, ParmaHamTopping $\dot{\sqsubseteq}$ MeatTopping\}$^*$\\
\{ParmaHamTopping $\dot{\sqsubseteq}$ FishTopping, ParmaHamTopping $\dot{\sqsubseteq}$ MeatTopping\}$^*$\\
\hline
\end{tabular}
\end{center}
\caption{Non-redundant repairing actions for MyPizza $\dot{\sqsubseteq}$ FishyMeatyPizza.}
\label{non-redundant-repairing-actions-single}
\end{figure}

In a similar way, for missing is-a relation MyFruttiDiMare $\dot{\sqsubseteq}$  NonVegetarianPizza, we find the following non-redundant coherence-preserving repairing actions: \{MyFruttiDiMare $\dot{\sqsubseteq}$ NonVegetarianPizza\}, \{AnchoviesTopping $\dot{\sqsubseteq}$ FishTopping\} and \{Anchovies\-Topping $\dot{\sqsubseteq}$ MeatTopping\}.
After combining the repairing actions for the individual missing is-a relations, and removing redundancy and incoherence-introducing repairing actions, we obtain the following 5 solutions for the repairing of the missing is-a structure:
\{MyPizza $\dot{\sqsubseteq}$ FishyMeatyPizza, MyFruttiDiMare $\dot{\sqsubseteq}$ NonVegetarianPizza\}, 
\{AnchoviesTopping $\dot{\sqsubseteq}$ FishTopping,  ParmaHamTopping $\dot{\sqsubseteq}$ MeatTopping\},
\{ParmaHamTopping $\dot{\sqsubseteq}$ FishTopping, AnchoviesTopping $\dot{\sqsubseteq}$ MeatTopping\},
\{MyPizza $\dot{\sqsubseteq}$ Fishy\-Meaty\-Pizza, AnchoviesTopping $\dot{\sqsubseteq}$ FishTopping\}, 
and \{MyPizza $\dot{\sqsubseteq}$ FishyMeaty\-Pizza, AnchoviesTopping $\dot{\sqsubseteq}$ MeatTopping\}.

A repairing action is a set of statements of the form $A$ $\dot{\sqsubseteq}$ $B$ with $A$ and $B$ named concepts.
We note that for acyclic terminologies, after adding such a statement, the resulting TBox is not an acyclic terminology anymore. If this is needed, then instead of adding $A$ $\dot{\sqsubseteq}$ $B$  the following should be done. If there is no definition for $A$ yet in the TBox, then add $A$ $\doteq$ $B$ $\sqcap$  $\overline{A}$ with $\overline{A}$ a new atomic concept. If there is already a definition for $A$ in the TBox, say $A$ $\doteq$ $C$ then change this definition to $A$ $\doteq$ $B$ $\sqcap$ $C$. For instance, to add the repairing action \{AnchoviesTopping $\dot{\sqsubseteq}$ FishTopping,  ParmaHamTopping $\dot{\sqsubseteq}$ MeatTopping\} to the acyclic terminology in Figure \ref{small-pizza-ontology-acyclic-terminology}, we have to change two definitions. 
AnchoviesTopping  $\doteq$ PizzaTopping  $\sqcap$ $\overline{AnchoviesTopping}$ becomes AnchoviesTopping  $\doteq$ PizzaTopping  $\sqcap$ $\overline{AnchoviesTopping}$ $\sqcap$ FishTopping, and ParmaHamTopping $\doteq$ PizzaTopping  $\sqcap$   $\overline{ParmaHamTopping}$ becomes ParmaHamTopping $\doteq$ PizzaTopping  $\sqcap$   $\overline{ParmaHamTopping}$ $\sqcap$  MeatTopping.

\subsection{Solution properties}

We prove that the algorithm is correct. In the following $A$ and $B$ are named concepts.

\begin{mylemma}\label{lemma-close}
If an ABox in the algorithm in Figure \ref{basic-algorithm} contains x:$A$ and x:$\lnot B$ for some individual x, then adding $A \sqsubseteq B$ to the ontology will close this ABox.
\end{mylemma}

\begin{proof}
Adding  $A \sqsubseteq B$ to an acyclic terminology can be done in the following way. If there is no definition for $A$ yet in the TBox, then add $A$ $\doteq$ $B$ $\sqcap$  $\overline{A}$ with $\overline{A}$ a new atomic concept. If there is already a definition for $A$ in the TBox, say $A$ $\doteq$ $C$ then change this definition to $A$ $\doteq$ $B$ $\sqcap$ $C$. In both cases it will allow an unfolding step in the ABox for x:$A$ such that x:$B$ $\sqcap$  $\overline{A}$
or x:$B$ $\sqcap$ $C$ is added to the ABox. A further application of the $\sqcap$-rule will then also add x:$B$ to the ABox. This leads to a contradiction with x:$\lnot B$ and closes the ABox.

  \hspace*{11cm} $\clubsuit$
\end{proof}

\begin{mylemma}
Let G be the completion graph after running the tableaux algorithm with $A_i \sqcap \lnot B_i$ as input on KB.
Let G$_e$ be the completion graph when running the tableaux algorithm with $A_i \sqcap \lnot B_i$ as input on KB$_{e}$, the knowledge base for the extended ontology where (i) a statement $A$ $\doteq$ $B$ $\sqcap$  $\overline{A}$ is added, or (ii) a statement $A$ $\doteq$ $C$ is changed to  $A$ $\doteq$ $B$ $\sqcap$ $C$.
Then for every open leaf ABox ${\cal A}_e$ in G$_e$, there is a corresponding open leaf ABox ${\cal A}$ in G such that ${\cal A}$ $\subseteq$ ${\cal A}_e$.
\end{mylemma}

\begin{proof}

We observe that the statements that are valid in the original ontology are still valid in the extended ontology. Also, all transformation rules that could be applied in the tableaux algorithm on the original ontology can also be applied on the extended ontology. 
Further, additionally, for the extended ontology, there may be an unfolding based on the extension (i.e. when x:$A$ appears in an ABox, then by unfolding x:$B$ (and in case (i) also x:$\overline{A}$) should be added), as well as additional applications of transformation rules based on x:$B$ and its consequences. 
Therefore, assume that ${\cal A}_e$ 
is an open leaf ABox in G$_e$. Then  ${\cal A}_e$ will contain statements that are a result of transformations that could have been peformed in the original ontology as well as statements that can only be derived based on x:$B$. 

In order to identify which statements in ${\cal A}_e$ belong to which of these two categories, we built a dependency graph {\cal DG} = (V, E) where the vertices represent statements from ${\cal A}_e$, and there is an edge from node $n_i$ related to statement $\tau_i$ to node $n_j$ related to statement $\tau_j$ if $\tau_j$ could be a direct result of applying some transformation rule or unfolding on $\tau_i$. We label the edges with $\sqcap$, $\sqcup$, $\exists_i$ (from a node related to x:$\exists$r.D to a node related to xry), $\exists_c$ (from a node related to x:$\exists$r.D to a node related to y:D), $\forall_i$ (from a node related to xry to a node related to y:C), $\forall_c$ (from a node related to x:$\forall$r.C to a node related to y:C) and $unfold$ depending on which transformation rule was applied.

Further, we tag the nodes in the dependency graph using $\sigma$ and $\theta$. When the tagging is finished, the statements related to the nodes tagged with $\sigma$ are a result of transformations that could have been performed in the original ontology. The algorithm for tagging nodes is given in Figure \ref{tagging-algorithm}.
Initially, the node related to x:$A_i \sqcap \lnot B_i$ is tagged with $\sigma$.
Further, for case (i) all nodes related to z:$\overline{A}$ and z:B $\sqcap$ $\overline{A}$ for some individual z  are tagged with $\theta$ and this tag can never be changed\footnote{$\overline{A}$ does not occur in the original ontology and thus these statements are not included in any ABox related to the original ontology.}; 
All other nodes are initially tagged with $\theta$ and collected in the set $\theta$-nodes.
We then traverse the dependency graph and change $\theta$ tags into $\sigma$ tags according to a number of rules. Essentially, the tagging simulates the construction of a completion graph (step 4) by first using  $\sqcap$, $\forall$, $\exists$ and unfolding rules until no more such rules are applicable (essentially an ABox in the completion graph). This is achieved in step 4.1. by applying the tagging rules 1-5 in Figure \ref{tagging-rules}.  For case (ii) we need to deal with a special case when unfolding x:A (step 4.1.2). When no more such rules can be applied, but a $\sqcup$-rule could be applied, a choice is made and the procedure is repeated (essentially a child ABox in the completion graph is created and then if needed, further descendant ABoxes). This is achieved by using tagging rules 6 and 7 in Figure \ref{tagging-rules} in step 4.4. Step 4 finishes when all nodes have $\sigma$ tags or no more tagging rules can be applied to nodes with $\theta$ tags.

\begin{figure}[tb]
\begin{center}
\scriptsize
\begin{tabular}{| l |}
\hline
A node $n$ tagged with $\theta$ changes its tag to $\sigma$ if: \\
1. node $n$ has an incoming $unfold$ edge from a $\sigma$-tagged node $p$\\
2. node $n$ has an incoming $\sqcap$ edge from a $\sigma$-tagged node $p$\\
3. node $n$ has an incoming $\exists_i$ edge from a $\sigma$-tagged node $p$\\
4. node $n$ has an incoming $\exists_c$ edge from a $\sigma$-tagged node $p$\\
5. node $n$ has an incoming $\forall_i$ edge from a $\sigma$-tagged node $p_i$ \\
\hspace*{0.5cm} and an incoming $\forall_c$ edge from a $\sigma$-tagged node $p_c$\\
6.  node $n$ has an incoming $\sqcup$ edge from a $\sigma$-tagged node $p$ \\  
\hspace*{0.5cm} and all other nodes with an incoming $\sqcup$ edge from $p$ are tagged with $\theta$\\
7.  node $n$ has an incoming $\sqcup$ edge from a $\sigma$-tagged node $p$ \\  
\hspace*{0.5cm} and there are no other nodes with an incoming $\sqcup$ edge from $p$\\
  \hline
\end{tabular}
\end{center}
\caption{Tagging rules.}
\label{tagging-rules}
\end{figure}

\begin{figure}[tb]
\begin{center}
\scriptsize
\begin{tabular}{| l |}
\hline
\noindent\textbf{\textit{Input}}:
The initial dependency graph $DG$.\\
\textit{\textbf{Output}}: Tagged dependency graph. \\
\textbf{\textit{Algorithm}} \\
1. Tag the node related to x:$A_i \sqcap \lnot B_i$ with $\sigma$.\\
1(i). Only for case (i): Tag all nodes related to z:$\overline{A}$ and z:B $\sqcap$ $\overline{A}$ for some indivdual z with $\theta$; \\ 
2. Tag all remaining nodes with $\theta$ and add them to the set $\theta$-nodes;  \\
3. $Visited$ := $\emptyset$; \\
4. {\bf While} $\theta$-nodes $\neq$ $\emptyset$ and $\theta$-nodes $\neq$ $Visited$, 
{\bf do}:\\
\hspace*{0.35cm} 4.1 {\bf While} $\theta$-nodes $\neq$ $\emptyset$ and $\theta$-nodes $\neq$ $Visited$, {\bf do}:\\
\hspace*{0.35cm}\hspace*{0.35cm} 4.1.1 Select a node  $n$ $\in$ $\theta$-nodes $\setminus$ $Visited$ and add $n$ to $Visited$; \\
\hspace*{0.35cm}\hspace*{0.35cm} 4.1.2 {\bf If} $n$ is related to x:B$\sqcap$C for some x \\ 
\hspace*{0.35cm}\hspace*{0.35cm}\hspace*{0.35cm} and rule 1 is applicable for unfolding a $\sigma$-tagged node related to x:A (only case (ii)\\
\hspace*{0.35cm}\hspace*{0.35cm}\hspace*{0.35cm} {\bf then} \\
\hspace*{0.35cm}\hspace*{0.35cm}\hspace*{0.35cm} 4.1.2.1 {\bf If} there is no node related to x:C \\
\hspace*{0.35cm}\hspace*{0.35cm}\hspace*{0.35cm}\hspace*{0.35cm} {\bf then} \\
\hspace*{0.35cm}\hspace*{0.35cm}\hspace*{0.35cm}\hspace*{0.35cm}
4.1.2.1.1 Add a new node related to x:C and connect it with outgoing $\sqcap$-edges to 
\\ \hspace*{0.35cm}\hspace*{0.35cm}\hspace*{0.35cm}\hspace*{0.35cm} all nodes that are connected with incoming $\sqcap$-edges to the node related to x:B$\sqcap$C, 
\\ \hspace*{0.35cm}\hspace*{0.35cm}\hspace*{0.35cm}\hspace*{0.35cm} except the node x:B;\\
\hspace*{0.35cm}\hspace*{0.35cm}\hspace*{0.35cm} 4.1.2.2
Tag the node related to x:C with $\sigma$;\\
\hspace*{0.35cm}\hspace*{0.35cm}\hspace*{0.35cm} {\bf else} \\
\hspace*{0.35cm}\hspace*{0.35cm}\hspace*{0.35cm} 4.1.2.1' {\bf If} any of the tagging rules 1-5 in Figure \ref{tagging-algorithm} is applicable, \\
\hspace*{0.35cm}\hspace*{0.35cm}\hspace*{0.35cm}\hspace*{0.35cm}{\bf then} \\
\hspace*{0.35cm}\hspace*{0.35cm}\hspace*{0.35cm}\hspace*{0.35cm} 4.1.2.1'.1 Change the tag of $n$ to $\sigma$; \\
\hspace*{0.35cm}\hspace*{0.35cm}\hspace*{0.35cm}\hspace*{0.35cm} 4.1.2.1'.2 Remove $n$ from $\theta$-nodes;\\
\hspace*{0.35cm}\hspace*{0.35cm}\hspace*{0.35cm}\hspace*{0.35cm} 4.1.2.1'.3 $Visited$ := $\emptyset$;\\ 


\hspace*{0.35cm} 4.2 $Visited$ := $\emptyset$;\\
\hspace*{0.35cm} 4.3 $Still$-$to$-$visit$ := true; \\
\hspace*{0.35cm} 4.4 {\bf While} $\theta$-nodes $\neq$ $\emptyset$ and $\theta$-nodes $\neq$ $Visited$ and $Still$-$to$-$visit$ = true, {\bf do}: \\
\hspace*{0.35cm}\hspace*{0.35cm} 4.4.1 Select a node  $n$ $\in$ $\theta$-nodes $\setminus$ $Visited$ and add $n$ to $Visited$; \\
\hspace*{0.35cm}\hspace*{0.35cm} 4.4.2 {\bf If} tagging rule 6 or 7 in Figure \ref{tagging-algorithm} is applicable, \\
\hspace*{0.35cm}\hspace*{0.35cm}\hspace*{0.35cm} {\bf then} \\
\hspace*{0.35cm}\hspace*{0.35cm}\hspace*{0.35cm}  4.4.2.1 Change the tag of $n$ to $\sigma$; \\
\hspace*{0.35cm}\hspace*{0.35cm}\hspace*{0.35cm}  4.4.2.2 Remove $n$ from $\theta$-nodes;\\
\hspace*{0.35cm}\hspace*{0.35cm}\hspace*{0.35cm}  4.4.2.3 $Visited$ := $\emptyset$;\\ 
\hspace*{0.35cm}\hspace*{0.35cm}\hspace*{0.35cm}  4.4.2.3 $Still$-$to$-$visit$ := false; \\ 

5. {\bf Return} the tagged graph;\\

  \hline
\end{tabular}
\end{center}
\caption{Algorithm for tagging the dependency graph.}
\label{tagging-algorithm}
\end{figure}

We now show that the set of statements related to $\sigma$-tagged nodes is an open leaf ABox ${\cal A}$ in G (with possible renaming of individuals). It is clear that ${\cal A}$ $\subseteq$ ${\cal A}_e$.
Further, ${\cal A}$ contains statements that can be derived by running the tableaux algorithm on the original ontology with $A_i \sqcap \lnot B_i$ as input. (Without loss of generality, individuals may be renamed.) The dependency graph shows which transformation rules and unfoldings can be performed.
${\cal A}$ is also a leaf ABox in the completion graph obtained by running  the tableaux algorithm on the original ontology with $A_i \sqcap \lnot B_i$ as input. If ${\cal A}$ were not a leaf ABox in that completion graph, then there would be other statements that could be added through the application of the transformation rules or unfolding. However, in that case, these transformation rules or unfoldings could also be applied to ${\cal A}_e$ and these statements could be added to ${\cal A}_e$, which would contradict the fact that ${\cal A}_e$ is a leaf ABox.
Finally, as  ${\cal A}$ $\subseteq$ ${\cal A}_e$ and as  ${\cal A}_e$ does not contain a contradiction (as it is an open leaf ABox in G$_e$), also ${\cal A}$ does not contain a contradiction.

 \hspace*{11cm} $\clubsuit$
\end{proof}

Consider the example ontology in Figure \ref{example-ontology-original} and assume we run the tableaux algorithm with input  x:A $\sqcap$ $\lnot$F. The completion graph is given in Figure \ref{Fig-G}. 
Figure \ref{Fig-Ge} shows a completion graph for running the tableaux algorithm with input  A $\sqcap$ $\lnot$F on the ontology that extends the original ontology with the axiom A $\doteq$ B $\sqcap$ $\overline{A}$ (Figure \ref{example-ontology-extended}).
The completion graph G$_e$  contains two open leaf ABoxes, ABox 1.2.1 and ABox 1.2.2. We observe that ABox 1.2.1 in Figure \ref{Fig-Ge} contains all literals from ABox 1.1 in Figure \ref{Fig-G} and ABox 1.2.2 in Figure \ref{Fig-Ge} contains all literals from ABox 1.2 in Figure \ref{Fig-G}.

Figure \ref{Fig-dependency} shows an initial dependency graph for ABox 1.2.1. In this example all but x:A $\sqcap \lnot$F can be the direct result of the application of a transformation rule or unfolding of only one other statement. In general this is not case, though. Initially, the node related to (1) is tagged as $\sigma$. All other nodes are tagged with $\theta$, and the tag of the nodes related to (3) and (5) cannot be changed. In the first iteration the tags of the nodes related to (2), (9) and (10) will be changed to $\sigma$ using the tagging rules 1 and 2. Further, the tag for the node related to (14) is changed to  $\sigma$ using tagging rule 7. Then, there will be no more changes in the tagging and the statements related to the nodes with $\sigma$ tags in Figure \ref{Fig-dependency-tagged} constitute the statements in ABox 1.1. in Figure \ref{Fig-G}.

\begin{figure}[tb]
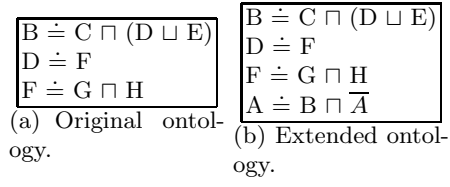

    \centering
    \subfigure[Original ontology.]
{
\begin{tabular}{| l |}
\hline
B  $\doteq$ C $\sqcap$ (D $\sqcup$ E) \\
D  $\doteq$ F \\
F  $\doteq$ G $\sqcap$ H \\
\hline
\end{tabular}
\label{example-ontology-original}
}
    \subfigure[Extended ontology.]
{
\begin{tabular}{| l |}
\hline
B  $\doteq$ C $\sqcap$ (D $\sqcup$ E) \\
D $\doteq$ F \\
F  $\doteq$ G $\sqcap$ H \\
A  $\doteq$ B $\sqcap$ $\overline{A}$ \\
\hline
\end{tabular}
\label{example-ontology-extended}
}
\label{example-ontology}
\caption{Ontology - original and extended.}
\end{figure}

\begin{figure}[tb]
    \centering
    \subfigure[Completion graph after running the tableaux algorithm with input A $\sqcap$ $\lnot$F on the original ontology in Figure \ref{example-ontology-original}.]
    {
        \includegraphics[width=0.47\textwidth]{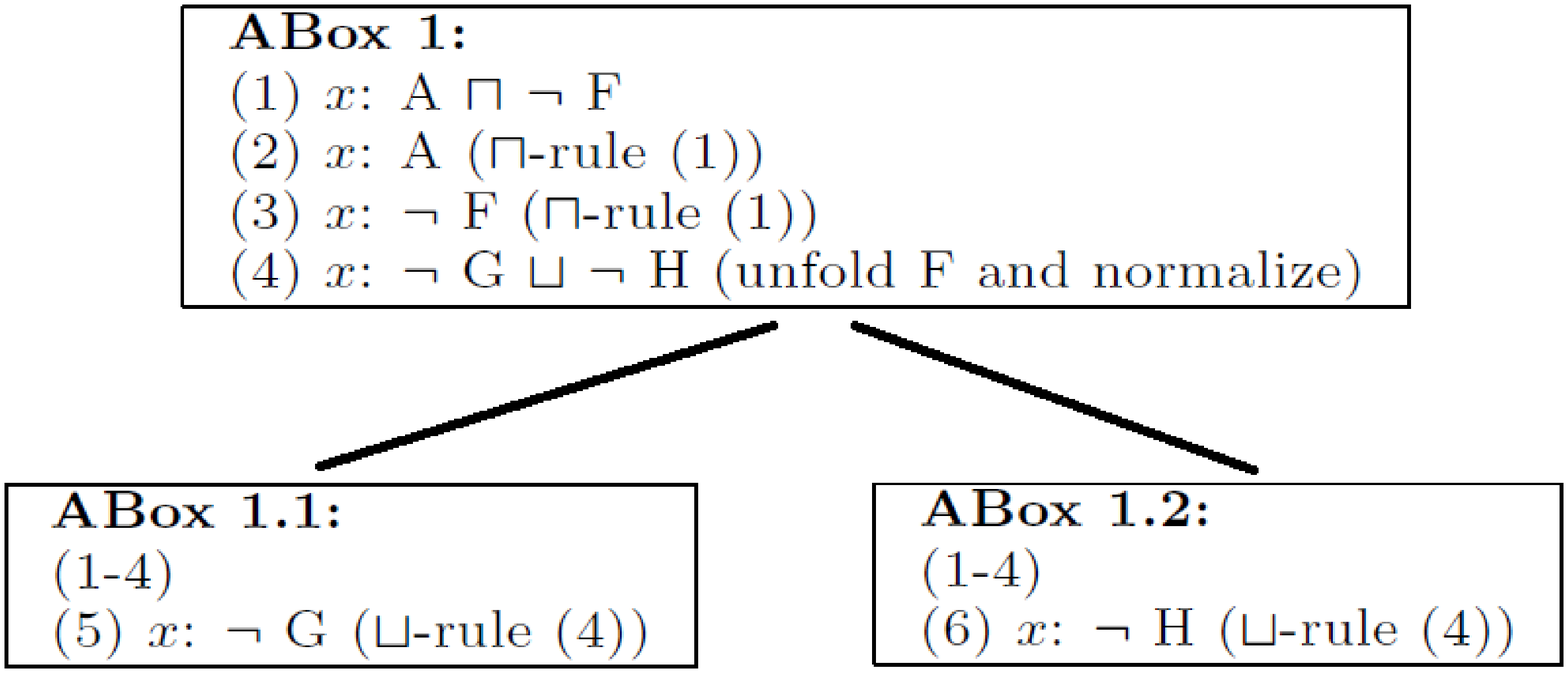}
        \label{Fig-G}
    }
    \subfigure[Completion graph after running the tableaux algorithm with input A $\sqcap$ $\lnot$F on the extended ontology in Figure \ref{example-ontology-extended}.]
    {
        \includegraphics[width=0.47\textwidth]{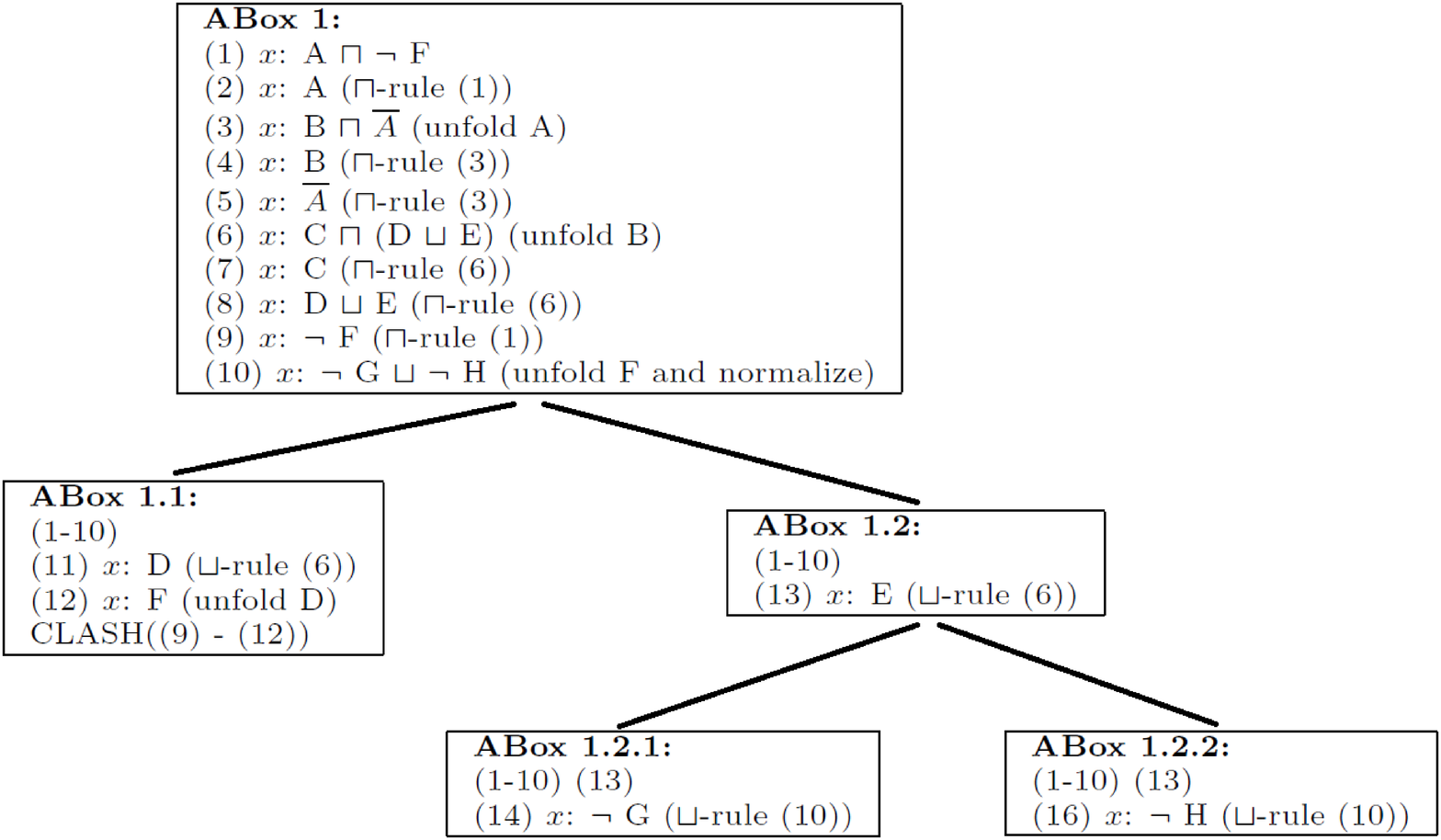}
        \label{Fig-Ge}
    }
    \label{fig:completion-graphs}
\caption{Completion graphs}.
\end{figure}

\begin{figure}[tb]
    \centering
    \subfigure[Initial dependency graph for ABox 1.2.1. in Figure \ref{Fig-Ge}.]
    {
        \includegraphics[width=0.47\textwidth]{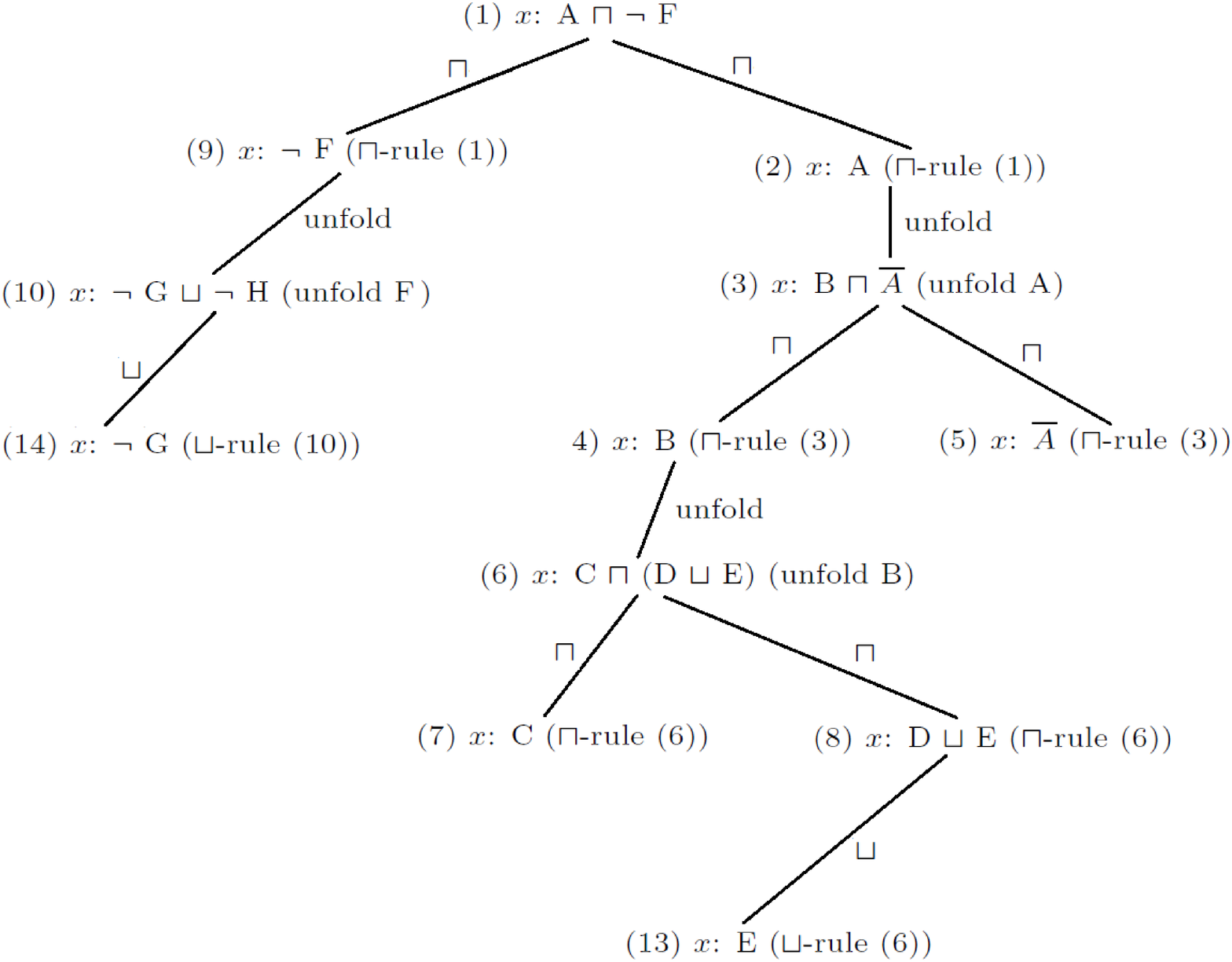}
        \label{Fig-dependency}
    }
    \subfigure[Tagged dependency graph for ABox 1.2.1. in Figure \ref{Fig-Ge}.]
    {
        \includegraphics[width=0.47\textwidth]{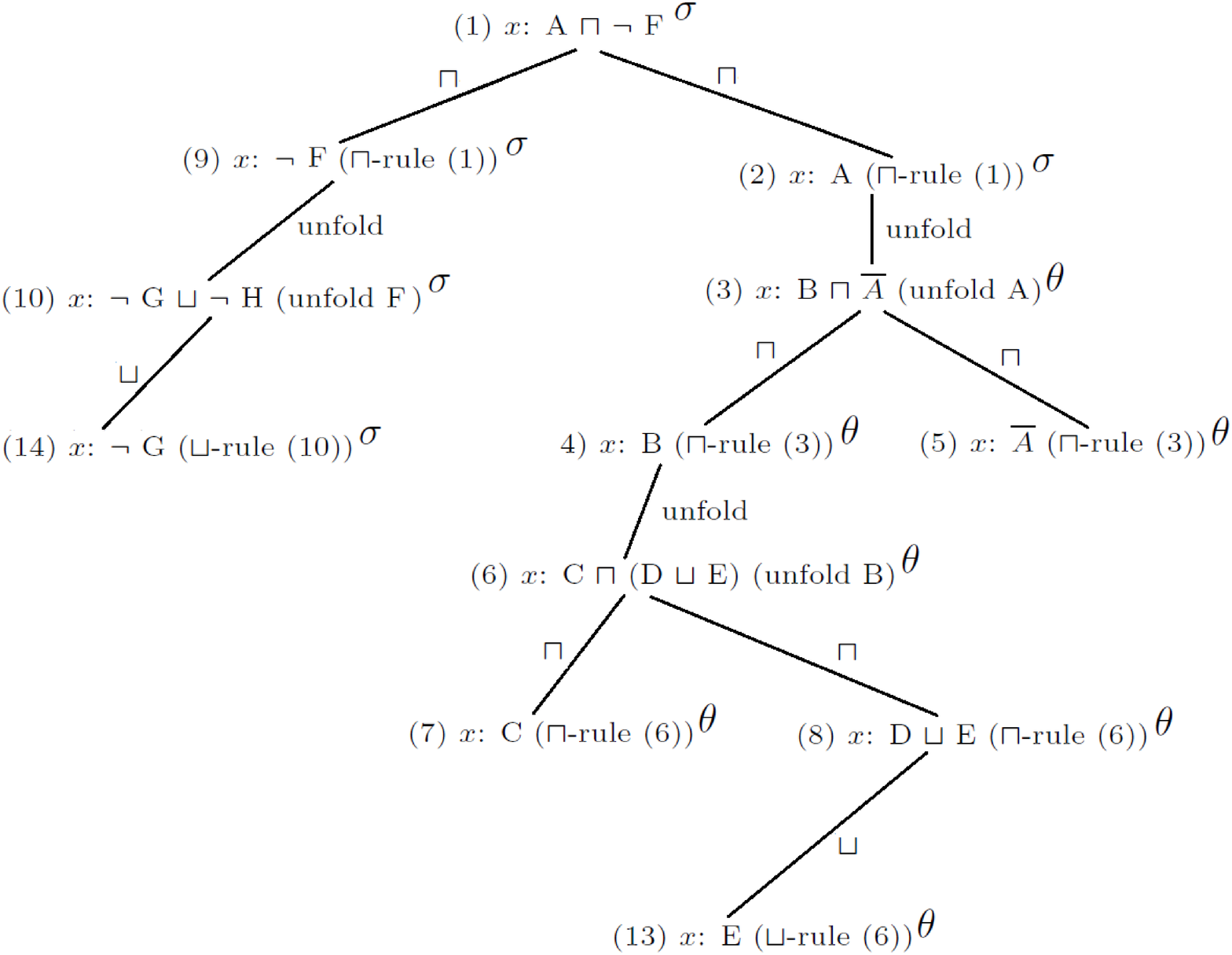}
        \label{Fig-dependency-tagged}
    }
    \label{fig:dependency-graphs}
\caption{Dependency graphs.}
\end{figure}

The previous lemma also holds when adding {\it sets} of axioms of the form (i) $A$ $\doteq$ $B$ $\sqcap$ $\overline{A}$ or changing a sets of axioms (ii) of the form $A$ $\doteq$ $C$ to axioms of the form $A$ $\doteq$ $B$ $\sqcap$ $C$.

\begin{mylemma}
Let G be the completion graph after running the tableaux algorithm with $A_i \sqcap \lnot B_i$ as input on KB.
Let G$_e$ be the completion graph when running the tableaux algorithm with $A_i \sqcap \lnot B_i$ as input on KB$_{e}$, the knowledge base for the extended ontology where (i)  statements of the form $A$ $\doteq$ $B$ $\sqcap$  $\overline{A}$ are added, or (ii) statements of the form $A$ $\doteq$ $C$ are changed to  $A$ $\doteq$ $B$ $\sqcap$ $C$.
Then for every open leaf ABox ${\cal A}_e$ in G$_e$, there is a corresponding open leaf ABox ${\cal A}$ in G such that ${\cal A}$ $\subseteq$ ${\cal A}_e$.

\end{mylemma}

\begin{proof}  

Let $\Theta$ = \{$\theta_1$, $\dots$,$\theta_n$\} be the changes to the original ontology, i.e.  $\theta_j$ represents either an addition of an axiom of the form $A$ $\doteq$ $B$ $\sqcap$ $\overline{A}$ or a change of an axiom of the form $A$ $\doteq$ $C$ to an axiom of the form $A$ $\doteq$ $B$ $\sqcap$ $C$.

For any possible order $\theta_1^\prime$, $\dots$ , $\theta_n^\prime$ of the elements in $\Theta$, let $\Theta^\prime_k$ = \{$\theta_1^\prime$, $\dots$,$\theta_k^\prime$\} $\subset$ $\Theta$  be the set of the first k changes according to the chosen order. Let KB$_e^{\Theta_k^\prime}$ be the knowledge base for the ontology that is constructed by applying the changes in $\Theta_k^\prime$ to the original ontology. We note that for k = n there is only one possible set of changes, i.e. 
$\Theta^\prime_n$ = $\Theta$ and  KB$_e^{\Theta_n}$ = KB$_{e}$.

According to the previous lemma, we then know that for any open leaf ABox in the completion graph obtained by running the tableaux algorithm with $A_i \sqcap \lnot B_i$ as input on  KB$_e^{\Theta_{k+1}^\prime}$ with k $<$ n, there is an open leaf ABox in the completion graph obtained by running the tableaux algorithm with $A_i \sqcap \lnot B_i$ as input on KB$_e^{\Theta_k^\prime}$, such that the latter is contained in the former. 

Therefore, for every open leaf ABox ${\cal A}_{e}$ in G$_e$, there is a chain of open leaf ABoxes ${\cal A}_{n-1}$, $\dots$, ${\cal A}_1$, ${\cal A}$, such that  ${\cal A}_k$ is an open leaf ABox in the completion graph obtained by running the tableaux algorithm with $A_i \sqcap \lnot B_i$ as input on KB$_e^{\Theta_k^\prime}$,  ${\cal A}$ is an open leaf ABox in G, and ${\cal A}$ $\subseteq$ ${\cal A}_1$, ${\cal A}_1$ $\subseteq$ ${\cal A}_2$, $\dots$ ${\cal A}_{n-1}$ $\subseteq$
 ${\cal A}_{e}$.

  \hspace*{11cm} $\clubsuit$
\end{proof}

\begin{mylemma}
Any element of $Rep(A_i,B_i)$ repairs the missing is-a relation $A_i$ $\dot{\sqsubseteq}$ $B_i$.
\end{mylemma}

\begin{proof}
Let \{$P_1$ $\dot{\sqsubseteq}$ $N_1$, $\dots$, $P_m$ $\dot{\sqsubseteq}$ $N_m$\} $\in$ $Rep(A_i,B_i)$. Let KB be the knowledge base of the original ontology. Then, we want to prove that  $A_i$ $\dot{\sqsubseteq}$ $B_i$ can be derived from the knowledge base of the ontology extended with \{$P_1$ $\dot{\sqsubseteq}$ $N_1$, $\dots$, $P_m$ $\dot{\sqsubseteq}$ $N_m$\}.

Adding \{$P_1$ $\dot{\sqsubseteq}$ $N_1$, $\dots$, $P_m$ $\dot{\sqsubseteq}$ $N_m$\} to an ontology means adding statements of the form 
$P_j$ $\doteq$ $N_j$ $\sqcap$ $\overline{P_j}$ with $\overline{P_j}$ a new atomic concept or changing statements of the form $P_j$ $\doteq$ $C$ to $P_j$ $\doteq$ $N_j$ $\sqcap$ $C$. Assume G and G$_e$ are the completion graphs  obtained by running the tableaux algorithm with $A_i \sqcap \lnot B_i$ as input on the knowledge bases of the original (with added missing is-a relations except $A_i \dot{\sqsubseteq} B_i$) and the extended ontology respectively.

We now prove that every leaf ABox in G$_e$ is closed.
Let us assume that there exists an open ABox ${\cal A}_e$ in G$_e$. According to the previous lemma, we know then that for ${\cal A}_e$ there is a corresponding leaf ABox ${\cal A}$ in G such that ${\cal A}$ $\subseteq$ ${\cal A}_e$. 
Further, there is at least one individual $x$ and named concepts $P$ and $N$ such that $x$:$P$ and $x$:$\neg$ $N$ are statements in the ABox ${\cal A}$\footnote{If the ABox is open then there must be at least one such situation. For instance, we know that $x$:$A_i$ and $x$:$\neg$ $B_i$ are in the ABox ${\cal A}$.}. $R_{\mathcal{A}}$ collects all possible is-a relations between such pairs $P$ and $N$. In the algorithm one such pair is chosen to be added to the ontology. Therefore, there is an individual $x$ and named concepts $P$ and $N$ such that $x$:$P$ and $x$:$\neg$ $N$ are statements in ${\cal A}_e$, and such that $P$ $\dot{\sqsubseteq}$ $N$ was added to the ontology. By lemma \ref{lemma-close} ${\cal A}_e$ is then closed which contradicts with our assumption.
This shows that all leaf ABoxes are closed and thus $A_i$ $\dot{\sqsubseteq}$ $B_i$ can be derived from the knowledge base of the ontology extended with \{$P_1$ $\dot{\sqsubseteq}$ $N_1$, $\dots$, $P_m$ $\dot{\sqsubseteq}$ $N_m$\}.

\hspace*{11cm} $\clubsuit$
\end{proof}

\begin{mytheorem}
Any element of $Rep(M)$ repairs the missing is-a relations in $M$.
\end{mytheorem}

\begin{proof}

The previous lemma guarantees that adding any element from $Rep(A_i,B_i)$ to KB allows us to derive $A_i$ $\dot{\sqsubseteq}$ $B_i$.
Therefore, adding one element from each $Rep(A_i,B_i)$ to KB guarantees that each missing is-a relation  $A_i$ $\dot{\sqsubseteq}$ $B_i$ is derivable.

\hspace*{11cm} $\clubsuit$
\end{proof}

The proposed algorithm returns minimal solutions for every missing is-relation. In our setting a solution is minimal if it contains only necessary information for repairing a missing is-a relation. 

\begin{mydef}\label{def-min}
Let KB be a knowledge base representing ontology O and $A$ $\dot{\sqsubseteq}$ $B$ a missing is-a relation. A repairing action $\{ C_1$ $\dot{\sqsubseteq}$ $D_1, \dots, C_n$ $\dot{\sqsubseteq}$ $D_n\}$ is said to be minimal if it holds that $KB \cup \{ C_1$ $\dot{\sqsubseteq}$ $D_1, \dots, C_n$ $\dot{\sqsubseteq}$ $D_n\} \models A$ $\dot{\sqsubseteq}$ $B$ and 
there is no $\{E_1$ $\dot{\sqsubseteq}$ $F_1, \dots, E_m$ $\dot{\sqsubseteq}$ $F_m \}$ 
$\subsetneq$
$\{ C_1$ $\dot{\sqsubseteq}$ $D_1, \dots, C_n$ $\dot{\sqsubseteq}$ $D_n\}$ such that
$KB \cup \{E_1$ $\dot{\sqsubseteq}$ $F_1, \dots, E_m$ $\dot{\sqsubseteq}$ $F_m \} \models A$ $\dot{\sqsubseteq}$ $B$. 
\end{mydef}

\begin{mytheorem}
The basic algorithm for generating repairing actions for missing is-a relations in Figure \ref{basic-algorithm} produces minimal repairing actions according to the Definition \ref{def-min} for each missing is-a relation in the set M.
\end{mytheorem}

\begin{proof}
Trivial. This kind of redundancy is removed in step 1.5.3 and step 4.

\hspace*{11cm} $\clubsuit$
 \end{proof}

\subsection{Optimization}
 
The basic algorithm may produce many redundant solutions already in step 1. For instance, if the root node in the completion graph contains $x$:$A$ and $x$:$\lnot B$ with $A$ and $B$ named concepts, then $A$ $\dot{\sqsubseteq}$ $B$ appears in every open leaf ABox as a possible way of closing that ABox.
To deal with this issue we modify the algorithm such that at each node with an open ABox it generates $Pos_{x_j}$ and $Neg_{x_j}$. However, $Pos_{x_j}$ and $Neg_{x_j}$ only contain the concepts that are not already in a  $Pos_{x_j}$ and $Neg_{x_j}$ related to an ancestor node. 
For instance, if $A$ $\in$ $Pos_x$ is related to the root node, then $A$ will not appear in any $Pos_x$ related to another node. For each ABox we then generate a set 
$R_{\mathcal{A}}$ = $\bigcup_x$ \{ P $\dot{\sqsubseteq}$  N \} where P $\in$ $Pos_x$ for the current node or any ancestor node,  N $\in$ $Neg_x$ for the current node or any ancestor node, and at least one of P and N occurs in the current node. 
This allows us to reduce the redundancy in step 1. An open leaf ABox can now be closed by using an element from $R_{\mathcal{A}}$ from the leaf node or from any ancestor node. When generating repairing actions in step 1.5 we then make sure that when an element related to a non-leaf node is chosen, that no additional element from any descendant of that node is selected. For instance, if any element from the root's $R_{\mathcal{A}}$ is chosen, then no other elements should be chosen.  

As an example, let us reconsider the computation of the repairing actions related to MyPizza $\dot{\sqsubseteq}$  Fishy\-MeatyPizza for the acyclic terminology in Figure \ref{small-pizza-ontology-acyclic-terminology}. In the optimized version of the algorithm, we compute $P_{x_j}$, $N_{x_j}$ and $R_{\mathcal{A}}$ for every open ABox (see Figure \ref{rep-optimized-algorithm}). The root ABox (ABox 1 in Figure \ref{fig:cgraph}) has statements of the forms $x_j$:$A$ and  $x_j$:$\neg$$A$ and we thus create Pos$_x$ = \{MyPizza, Pizza\},  Neg$_x$ = \{FishyMeatyPizza\}, Pos$_y$ = \{AnchoviesTopping, $\overline{AnchoviesTopping}$, PizzaTopping\},  Neg$_y$ = $\emptyset$, Pos$_z$ = \{ParmaHamTopping, $\overline{ParmaHamTopping}$, PizzaTopping\},  Neg$_z$ = $\emptyset$. This leads to $R_{\mathcal{A}1}$ = \{MyPizza $\dot{\sqsubseteq}$ FishyMeatyPizza, Pizza $\dot{\sqsubseteq}$ FishyMeatyPizza\} for ABox 1. We also know now that any element of the $R_{\mathcal{A}1}$ will close all leaf nodes. For Abox 1.2  the new statements of the form $x_j$:$A$ and  $x_j$:$\neg$$A$ (i.e. not occurring in an ancestor) are $y$:$\neg$FishTopping and $z$:$\neg$FishTopping. Therefore, we create Pos$_x$ =  $\emptyset$,  Neg$_x$ = $\emptyset$, Pos$_y$ = $\emptyset$,  Neg$_y$ = \{FishTopping\}, Pos$_z$ = $\emptyset$,  Neg$_z$ = \{FishTopping\} for Abox 1.2. $R_{\mathcal{A}1.2}$  for Abox 1.2 contains the new ways to close this ABox (i.e. ways not occurring in ancestor nodes) and contains AnchoviesTopping $\dot{\sqsubseteq}$ FishTopping,
PizzaTopping $\dot{\sqsubseteq}$ FishTopping, and
ParmaHamTopping $\dot{\sqsubseteq}$ FishTopping. We now know that any of these will close all leaf ABoxes of Abox 1.2. After all $R_{\mathcal{A}}$ are computed for all open ABoxes, a leaf node can be closed using an element from its $R_{\mathcal{A}}$ or any $R_{\mathcal{A}}$ related to an ancestor node. For instance, ABox 1.2.2.2 can be closed using any element from $R_{\mathcal{A}1.2.2.2}$, $R_{\mathcal{A}1.2.2}$, $R_{\mathcal{A}1.2}$ and $R_{\mathcal{A}1}$.
When creating repairing actions we then make sure that when an element related to a non-leaf node is chosen, that no additional element from any descendant of that node is selected to close any leaf ABoxes that are descendants of that non-leaf node.  For instance, if AnchoviesTopping $\dot{\sqsubseteq}$ FishTopping related to ABox 1.2 is chosen, then no additional element is chosen to close any leaf ABoxes that are descendants of Abox 1.2.

\subsection{Extension}

The algorithm can be extended to generate additional repairing actions for every individual missing is-a relation. In step 1.5 if $A$ $\dot{\sqsubseteq}$ $B$ is used as one of the is-a relations in a repairing action then also $S$ $\dot{\sqsubseteq}$ $T$ where $S$ is a super-concept of $A$ and $T$ is a sub-concept of $B$ could be used. Therefore, the extended algorithm generates two sets of concepts for every is-a relation $A$ $\dot{\sqsubseteq}$ $B$ in a repairing action, Source set containing named super-concepts of $A$ and Target set containing named sub-concepts of $B$. Further, to not introduce non-validated equivalence relations where in the original ontology there are only is-a relations, we remove the super-concepts of $B$ from Source, and the sub-concepts of $A$ from Target.
 
Alternative repairing actions for a repairing action \{$A_1$ $\dot{\sqsubseteq}$ $B_1$,  $\dots$, $A_n$ $\dot{\sqsubseteq}$ $B_n$\} are then repairing actions \{$S_1$ $\dot{\sqsubseteq}$ $T_1$,  $\dots$, $S_n$ $\dot{\sqsubseteq}$ $T_n$\} such that $(S_i, T_i) \in Source(A_i, B_i) \times Target(A_i, B_i)$. This extension allows the algorithm to produce more informative repairing actions. 

Next, we prove the correctness of the proposed extension.
\begin{mytheorem}
If a missing is-a relation $A$ $\dot{\sqsubseteq}$ $B$ is repaired by a repairing action $\{C_1$ $\dot{\sqsubseteq}$ $D_1, $ $\dots$ $, C_n$ $\dot{\sqsubseteq}$ $D_n\}$ then $ A$ $\dot{\sqsubseteq}$ $B$ will also be repaired by $\{S_i$ $\dot{\sqsubseteq}$ $T_i | \forall i:1..n: S_i \in Source(C_i, D_i) \wedge T_i \in Target(C_i, D_i)\}$.
\end{mytheorem}

\begin{proof}
Let KB be the knowledge base of the original ontology. As  $\{C_1$ $\dot{\sqsubseteq}$ $D_1, $ $\dots$, $ C_n$ $\dot{\sqsubseteq}$ $D_n\}$ is a repairing action for $A$ $\dot{\sqsubseteq}$ $B$, we know that KB $\cup$ $\{C_1$ $\dot{\sqsubseteq}$ $D_1, $ $\dots$, $ C_n$ $\dot{\sqsubseteq}$ $D_n\}$  $\models$  $A$ $\dot{\sqsubseteq} B$.

To prove that $A$ $\dot{\sqsubseteq}$ $B$ is repaired by $\{S_i$ $\dot{\sqsubseteq}$ $T_i | \forall i:1..n: S_i \in Source(C_i, D_i) \wedge T_i \in Target(C_i, D_i)\}$, we need to show that $A$ $\dot{\sqsubseteq}$ $B$ can be derived from the knowledge base of the extended ontoloy KB $\cup$ $\{S_i$ $\dot{\sqsubseteq}$ $T_i | \forall i:1..n: S_i \in Source(C_i, D_i) \wedge T_i \in Target(C_i, D_i)\}$.

As $S_i \in Source(C_i, D_i)$, we know that KB $\models$ $C_i$ $\dot{\sqsubseteq}$ $S_i$. Further as $T_i \in Target(C_i, D_i)$, we know that KB $\models$ $T_i$ $\dot{\sqsubseteq}$ $D_i$. 
Therefore, 
KB  $\cup$ $\{S_i$ $\dot{\sqsubseteq}$ $T_i | \forall i:1..n: S_i \in Source(C_i, D_i) \wedge T_i \in Target(C_i, D_i)\}$ $\models$  $C_i$ $\dot{\sqsubseteq}$ $S_i$ $\wedge$  $S_i$ $\dot{\sqsubseteq}$ $T_i$ $\wedge$ $T_i$ $\dot{\sqsubseteq}$ $D_i$, and therefore, KB  $\cup$ $\{S_i$ $\dot{\sqsubseteq}$ $T_i | \forall i:1..n: S_i \in Source(C_i, D_i) \wedge T_i \in Target(C_i, D_i)\}$ $\models$ $C_i$ $\dot{\sqsubseteq}$ $D_i$.

This shows that KB  $\cup$ $\{S_i$ $\dot{\sqsubseteq}$ $T_i | \forall i:1..n: S_i \in Source(C_i, D_i) \wedge T_i \in Target(C_i, D_i)\}$ entails all statements that KB $\cup$ $\{C_1$ $\dot{\sqsubseteq}$ $D_1, $ $\dots$, $ C_n$ $\dot{\sqsubseteq}$ $D_n\}$ entails and thus also $A$ $\dot{\sqsubseteq}$ $B$.

\hspace*{11cm} $\clubsuit$
\end{proof} 

As an example, consider the repairing action 
\{AnchoviesTopping $\dot{\sqsubseteq}$ FishTopping,  ParmaHamTopping $\dot{\sqsubseteq}$ MeatTopping\} for the missing is-a relations M = \{MyPizza $\dot{\sqsubseteq}$ FishyMeatyPizza, MyFruttiDiMare $\dot{\sqsubseteq}$ NonVegetarianPizza\} for the ontology in Figure \ref{small-pizza-ontology-acyclic-terminology}, as computed using the basic algorithm.
The Source set for AnchoviesTopping $\dot{\sqsubseteq}$ FishTopping contains all named super-concepts of AnchoviesTopping that are not super-concepts of FishTopping, i.e. \{AnchoviesTopping, PizzaTopping\} $\setminus$ \{FishTopping, PizzaTopping\} = \{AnchoviesTopping\}. 
The Target set for AnchoviesTopping $\dot{\sqsubseteq}$ FishTopping contains all named sub-concepts of FishTopping that are not sub-concepts of AnchoviesTopping, i.e. \{FishTopping\} $\setminus$ \{AnchoviesTopping\} = \{FishTopping\}.
For ParmaHamTopping $\dot{\sqsubseteq}$ MeatTopping the Source set is \{ParmaHamTopping, PizzaTopping\} $\setminus$ \{MeatTopping, PizzaTopping\} = \{ParmaHamTopping\}, and the Target set is \{MeatTopping, HamTopping\} $\setminus$ \{ParmaHamTopping\} = \{MeatTopping, HamTopping\}. In this small example, using Source and Target sets, we would obtain one additional repairing action \{AnchoviesTopping $\dot{\sqsubseteq}$ FishTopping,  ParmaHamTopping $\dot{\sqsubseteq}$ HamTopping\}.

\section{Implementation}
\label{section-implementation}

\begin{figure}[tb]
\begin{center}
	\includegraphics[width=1\textwidth]{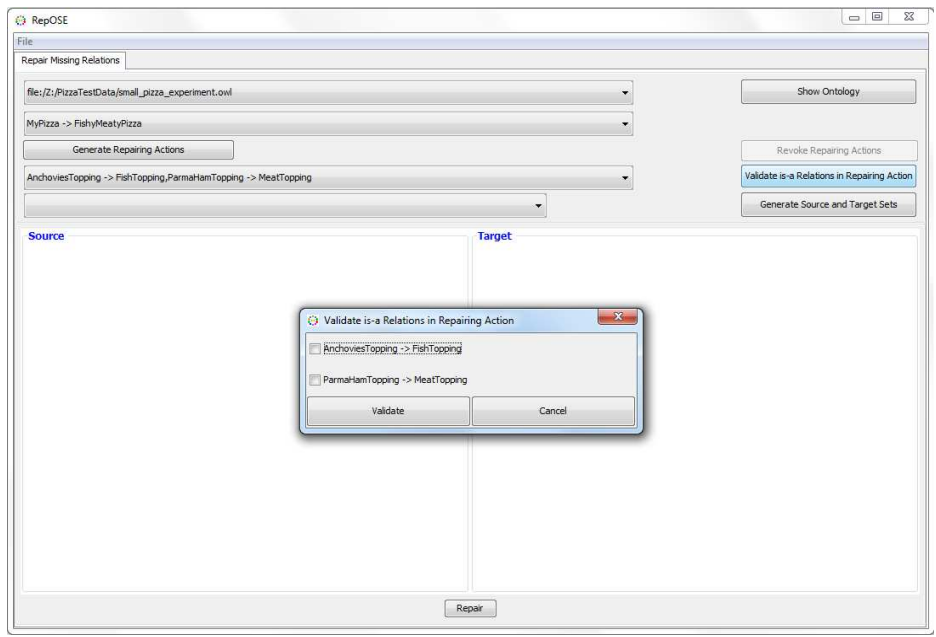}
        \caption{Screenshot - Validating is-a relations in a repairing action.}
	\label{Fig-Validate}
\end{center}
\end{figure}

\begin{figure}[tb]
\begin{center}
	\includegraphics[width=1\textwidth]{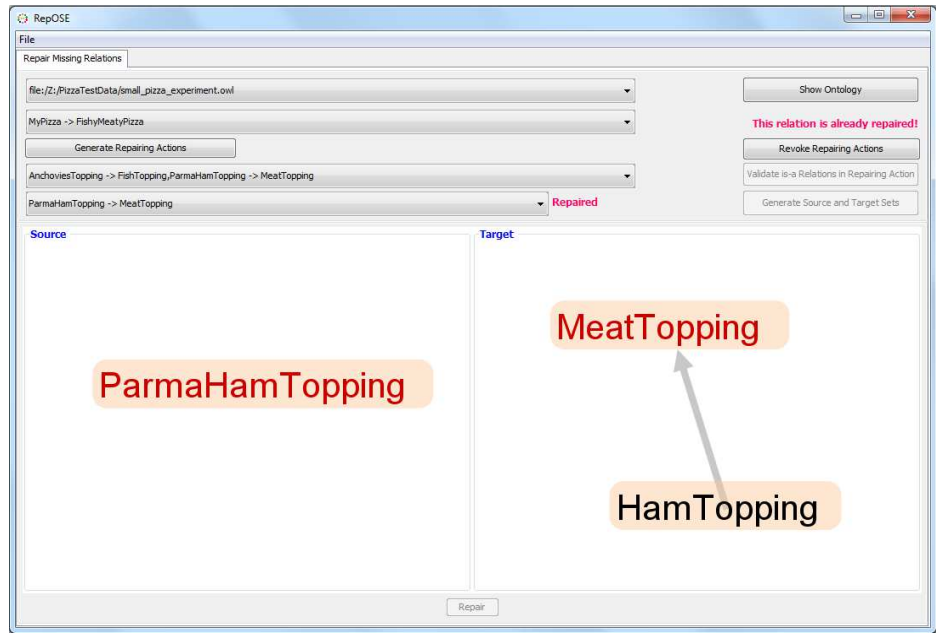}
        \caption{Screenshot - Repairing using Source and Target sets.}
	\label{Fig-Source-Target}
\end{center}
\end{figure}

We have implemented a system that supports the user to repair missing is-a relations. In our system the user loads the ontology and the missing is-a relations from the 'File' menu. The missing is-a relations are then shown in a drop-down list (e.g. MyPizza $\dot{\sqsubseteq}$ FishyMeatyPizza\footnote{In the system C $\dot{\sqsubseteq}$ D is shown as C $\rightarrow$ D.} in the second drop-down list in Figure \ref{Fig-Validate}). The user then chooses a missing is-a relation to repair.\footnote{As we repair one is-a relation at a time, there may be some redundancy in the solutions.} At any time the user can switch between different missing is-a relations. 
Once a missing is-a relation is chosen for repairing the user generates repairing actions for it by clicking the 'Generate Repairing Actions' button. 
This covers step 1 in Figure \ref{basic-algorithm} and was implemented in Java using Pellet (version 2.3.0) \cite{sirin2007}.
The satisfiability checker in Pellet was modified in order to extract full completion graphs. Furthermore, to increase performance and account for higher level of non-determinism, ontologies are first passed through Pellint \cite{Lin08} before running the algorithm. 
The computed repairing actions are shown in the drop-down list under the button. Each repairing action consists of one or more is-a relations. In Figure \ref{Fig-Validate} the user has chosen to repair MyPizza $\dot{\sqsubseteq}$ FishyMeatyPizza and the system has generated three repairing actions that do not introduce incoherence (\{MyPizza $\dot{\sqsubseteq}$ FishyMeatyPizza\}, \{AnchoviesTopping $\dot{\sqsubseteq}$ FishTopping, ParmaHamTopping $\dot{\sqsubseteq}$ MeatTopping\}, and \{AnchoviesTopping $\dot{\sqsubseteq}$ MeatTopping, ParmaHamTopping $\dot{\sqsubseteq}$ FishTopping\}). 
To repair this missing is-a relation the user has to succesfully deal with at least one of the repairing actions, i.e. add all is-a relations (or more informative is-a relations from Source and Target) in at least one of the repairing actions to the ontology.
In Figure \ref{Fig-Validate} the user has chosen the second repairing action. When repairing actions are added to the ontology, they will make the missing is-a relation derivable. However, it is not guaranteed that all the is-a relations in the repairing actions are also valid with respect to the domain. Therefore, a domain expert needs to validate the is-a relations in the repairing actions. When the user clicks the 'Validate is-a Relations in Repairing Action' button, a pop-up window (Figure \ref{Fig-Validate}) appears where the user can mark all the is-a relations that are correct with respect to the domain model. The repairing actions for all missing is-a relations and the ontology are updated according to the results of the validation. If an is-a relation is validated as incorrect according to the domain, all repairing actions that contain this incorrect is-a relation, for this and for all other missing is-a relations, are removed from the lists of the repairing actions. When an is-a relation is validated as correct it is added to the ontology and it is marked as correct in all repairing actions for all missing is-a relations. 
When all is-a relations in the current repairing action are validated as correct, they are shown in the last drop-down list (Figure \ref{Fig-Source-Target}).
Now the user can repair them one by one.

For each is-a relation within the repairing action the Source and Target sets are generated and displayed on the left and the right hand sides, respectively, within the panel under the drop-down lists (Figure \ref{Fig-Source-Target}). Both panels have zoom control and can be opened in a separate window. The concepts in the is-a relation under consideration are highlighted in red, existing asserted and inferred is-a relations are shown in grey, not yet repaired missing is-a relations in blue and is-a relations that were previously added for repairing missing is-a relations in black. In order to repair the is-a relation the user has to choose one concept from each of the sets and click the 'Repair' button.
In Figure \ref{Fig-Source-Target} the user has chosen to repair ParmaHamTopping $\dot{\sqsubseteq}$ MeatTopping with ParmaHamTopping $\dot{\sqsubseteq}$ HamTopping.
Upon repair the ontology is updated, i.e. the chosen is-a relation (ParmaHamTopping $\dot{\sqsubseteq}$ HamTopping) is added to the ontology. A red label next to the drop-down list shows the status (Repaired or Not Repaired) of the selected is-a relation. When all is-a relations within a repairing action are repaired the missing is-a relation is marked as repaired ('This relation is already repaired' label in Figure \ref{Fig-Source-Target}).
The other repairing actions are still available for review by the user. These may give information about other possible missing is-a relations.
The user can also revoke repairing actions (through the 'Revoke Repairing Actions' button). If the user revokes the repairing action, the missing is-a relation may become not repaired again and the is-a relations within the repairing action are marked as not repaired. All changes in the ontology are revoked and the user can start repairing this missing is-a relation again in the way just described.

\vspace*{-0.5cm}

\subsubsection{Example run}
\label{section-example}

As an example run, consider the ontology in Figure \ref{small-pizza-ontology}
and missing is-a relations MyPizza $\dot{\sqsubseteq}$  FishyMeatyPizza and MyFruttiDiMare $\dot{\sqsubseteq}$ NonVegetarianPizza. After loading the ontology and the missing is-a relations, we can choose a missing is-a relation to repair. Assume we choose MyFruttiDiMare $\dot{\sqsubseteq}$ NonVegetarianPizza and click the 'Generate Repairing Actions' button. The system will show three repairing actions in the drop-down list: 
\{MyFruttiDiMare $\dot{\sqsubseteq}$ NonVegetarianPizza\}, \{AnchoviesTopping  $\dot{\sqsubseteq}$ FishTopping\}, and \{AnchoviesTopping $\dot{\sqsubseteq}$ MeatTopping\}. We can choose to deal with AnchoviesTopping $\dot{\sqsubseteq}$ MeatTopping and validate this to be incorrect with respect to the domain. In this case all repairing actions containing this is-a relation will be removed. We could then choose AnchoviesTopping  $\dot{\sqsubseteq}$ FishTopping and validate it to be correct. All is-a relations in this repairing action (i.e.  AnchoviesTopping  $\dot{\sqsubseteq}$ FishTopping)  are validated to be correct and thus we can continue with the repair of AnchoviesTopping  $\dot{\sqsubseteq}$ FishTopping. In this small example the Source set only contains AnchoviesTopping and the Target set only contains FishTopping. Therefore, we click on these concepts and the 'Repair' button.  AnchoviesTopping  $\dot{\sqsubseteq}$ FishTopping will be marked as repaired and also the missing is-a relation MyFruttiDiMare $\dot{\sqsubseteq}$ NonVegetarianPizza will be marked as repaired.

We can then start repairing  MyPizza $\dot{\sqsubseteq}$  FishyMeatyPizza. The system would have generated as repairing actions that do not introduce incoherence \{MyPizza $\dot{\sqsubseteq}$  FishyMeatyPizza\}, \{Anchovies\-Topping $\dot{\sqsubseteq}$ FishTopping, ParmaHamTopping $\dot{\sqsubseteq}$ MeatTopping\} and \{Anchovies\-Topping $\dot{\sqsubseteq}$ MeatTopping, ParmaHamTopping $\dot{\sqsubseteq}$ FishTopping\}. However, as we earlier already validated  AnchoviesTopping $\dot{\sqsubseteq}$ MeatTopping to be incorrect with respect to the domain, the third repairing action has already been removed. When we choose the second repairing action, AnchoviesTopping $\dot{\sqsubseteq}$ FishTopping is already marked as correct (because of earlier validation) and only ParmaHamTopping $\dot{\sqsubseteq}$ MeatTopping needs to be validated. We validate this as correct and then choose to repair it. The Source set in this small example contains ParmaHamTopping and the Target set contains HamTopping and MeatTopping. Although we can add ParmaHamTopping $\dot{\sqsubseteq}$ MeatTopping, it is more informative (and correct with respect to the domain) to add ParmaHamTopping $\dot{\sqsubseteq}$ HamTopping. We therefore choose the latter. All is-a relations in this repairing action are then repaired and thus also MyPizza $\dot{\sqsubseteq}$  FishyMeatyPizza.

\section{Related work}
\label{section-related-work}

\subsubsection{Debugging ontologies}
{\it Detecting} missing is-a relations can be done in a number of ways (see Section \ref{section-introduction}). There is, however, not much work on the {\it repairing} of missing is-a structure. In \cite{LLT09} we addressed this in the setting of taxonomies. 

There is more work on the debugging of semantic defects.
Most of it aims at identifying and removing logical contradictions from an ontology. Standard reasoners are used to identify the existence of a contradiction, and provide support for resolving and eliminating it \cite{FMKPA08}. 
In \cite{Sch05} minimal sets of axioms are identified which need to be removed to render an ontology coherent. In \cite{KPSH06,KPSC06} strategies are described for repairing unsatisfiable concepts detected by reasoners, explanation of errors, ranking erroneous axioms, and generating repair plans. In \cite{HS05} the focus is on maintaining the consistency as the ontology evolves through a formalization of the semantics of change for ontologies.
In \cite{MST07} and \cite{JHQHS09} the setting is extended to repairing ontologies connected by mappings. In this case, semantic defects may be introduced by integrating ontologies. Both works assume that ontologies are more reliable than the mappings and try to remove some of the mappings to restore consistency. The solutions are often based on the computation of minimal unsatisfiability-preserving sets or minimal conflict sets. The work in \cite{QJH09} further characterizes the problem as mapping revision. Using belief revision theory, the authors give an analysis for the logical properties of the revision algorithms. Another approach for debugging mappings is proposed in \cite{WX08} where the authors focus on the detection of certain kinds of defects and redundancy.
The approach in \cite{JCHB09} deals with the inconsistencies introduced by the integration of ontologies, and unintended entailments validated by the user. 

Work that deals with both modeling and semantic defects includes \cite{CRVP09} where the authors propose an approach for detecting modeling and semantic defects within an ontology based on patterns and antipatterns. Some suggestions for repairing are also given. In \cite{LL11} we provided a method to detect and repair wrong and missing is-a structure in taxonomies connected by mappings. 

A different setting is the area of modular ontologies where the ontologies are connected by directional mappings and where knowledge propagation only occurs in one direction. Regarding the detection of semantic defects, within a framework based on distributed description logics, it is possible to restrict the propagation of local inconsistency to the whole set of ontologies (e.g. \cite{SBT05}). 

\vspace*{-0.5cm}

\subsubsection{Abductive reasoning in description logics}
In \cite{elsenbroich2006case} four different abductive reasoning tasks are defined - (conditionalized) concept abduction, ABox abduction, TBox abduction and knowledge base abduction. Concept abduction deals with finding sub-concepts. Abox abduction deals with retrieving abductively instances of concepts or roles that, when added to the knowledge base, allow the entailment of a desired ABox assertion. Knowledge base abduction includes both ABox and TBox abduction.
 
Most existing approaches for DL abduction focus on ABox and concept abduction and are mostly based on existing proof techniques such as semantic tableaux and resolution. Since the number of possible solutions is very large, the approaches introduce constraints. The work in \cite{klarman2011} proposes a goal-oriented approach where only actions which contribute to the solution are chosen in the proof procedures. 
The method is both complete and sound for consistent and semantically minimal solutions. Since the set of solutions can contain some inconsistent and non-minimal solutions additional checks are required.
A practical approach for ABox abduction, based on abductive logic programming, was proposed in \cite{du2011towards}. In order to use existing abductive logic programming systems it is necessary to do a transformation to a plain Datalog program. 
The solutions are consistent and minimal given a set of abduciles. However, the approach does not guarantee completeness since the translation to a Datalog program is approximate and in some cases a solution would not be found.
An approach for conditionalised concept abduction that uses a variation of the semantic tableaux and two labeling functions was proposed in \cite{colucci2004uniform}. The two labeling functions $T()$ and $F()$ represent true and false formulas in a tableaux where the solutions are formed from concepts which have at least one constraint in F() of every open branch. This choice is non-deterministic and can be used to select solutions based on some criteria. The algorithm also contains a consistency check which implies that produced solutions are always consistent.

There has not been much work related to TBox abduction, which is the most relevant abduction problem for this paper. The work in \cite{hubauer2010automata} proposes an automata-based approach to TBox abduction using abduciles representing axioms that can appear in solutions. It is based on a reduction to the axiom pinpointing problem which is then solved with automata-based methods. A PTIME algorithm is proposed for the language  $\mathcal{EL}$.

All of the presented approaches to description logic abduction work with relatively inexpressive ontologies, such as $\mathcal{EL}$ and $\mathcal{ALC}$. However, some recent work \cite{di2009tableaux} describes a type of conditionalised concept abduction called structural abduction which is applicable to $\mathcal{SH}$.

\section{Conclusion}
\label{section-conclusion}

This paper formalized repairing the is-a structure in ${\cal ALC}$ acyclic terminologies as a generalized TBox abduction problem, provided a solution based on semantic tableaux, and discussed a system.

There are a number of interesting aspects for future work. First, we intend to extend the algorithm to deal with more expressive ontologies. Further, it may be useful to consider also solutions introducing incoherence as they may lead to the detection of other kinds of modeling defects such as wrong is-a relations. In this case we do not assume anymore that the existing structure is correct. As a domain expert may need to deal with many possible solutions, other useful extensions are mechanisms for ranking the generated repairing actions and among those recommending repairing actions, e.g. based on domain knowledge.

\bibliographystyle{plain}
\bibliography{jist}

\begin{figure}
\begin{center}
	\includegraphics[angle=90,width=0.9\textwidth]{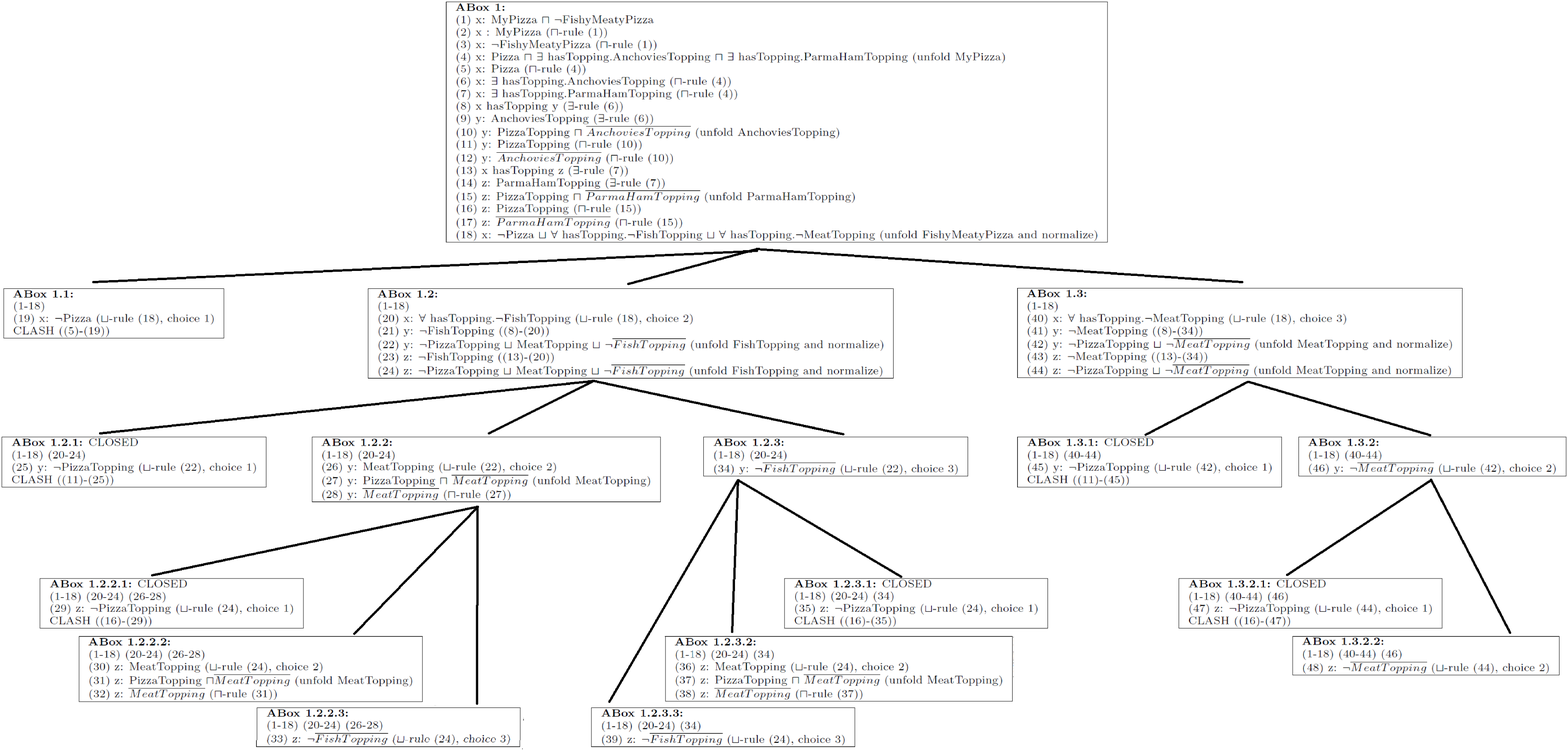}
\end{center}
\caption{Completion graph for $MyPizza \sqcap \lnot FishyMeatyPizza$.}
	\label{fig:cgraph}
\end{figure}

\begin{figure}[tb]
\begin{center}
\scriptsize
\begin{tabular}{| l |}
\hline
{\bf ABox 1.1}: CLOSED\\
\hline 
{\bf ABox 1.2.1}: CLOSED\\
\hline
{\bf ABox 1.2.2.1}: CLOSED\\
\hline
{\bf ABox 1.2.2.2}: \\
Pos$_x$ = \{MyPizza, Pizza\},\\
 Neg$_x$ = \{FishyMeatyPizza\} \\
Pos$_y$ = \{AnchoviesTopping, $\overline{AnchoviesTopping}$, PizzaTopping, MeatTopping, $\overline{MeatTopping}$\}, \\
 Neg$_y$ = \{FishTopping\} \\
Pos$_z$ = \{ParmaHamTopping, $\overline{ParmaHamTopping}$, PizzaTopping, MeatTopping, $\overline{MeatTopping}$\}, \\
Neg$_z$ = \{FishTopping\} \\
$R_{\mathcal{A}}$ = \{MyPizza $\dot{\sqsubseteq}$ FishyMeatyPizza, Pizza $\dot{\sqsubseteq}$ FishyMeatyPizza, 
AnchoviesTopping $\dot{\sqsubseteq}$ FishTopping, \\ \hspace*{0.8cm}
PizzaTopping $\dot{\sqsubseteq}$ FishTopping,
ParmaHamTopping $\dot{\sqsubseteq}$ FishTopping,
MeatTopping $\dot{\sqsubseteq}$ FishTopping\}\\
\hline 
{\bf ABox 1.2.2.3}: \\
Pos$_x$ = \{MyPizza, Pizza\}, \\ Neg$_x$ = \{FishyMeatyPizza\} \\
Pos$_y$ = \{AnchoviesTopping, $\overline{AnchoviesTopping}$, PizzaTopping, MeatTopping, $\overline{MeatTopping}$\}, \\ Neg$_y$ =  \{FishTopping\} \\
Pos$_z$ = \{ParmaHamTopping, $\overline{ParmaHamTopping}$, PizzaTopping\}, \\ Neg$_z$ =  \{FishTopping, $\overline{FishTopping}$\}\\
$R_{\mathcal{A}}$ = \{MyPizza $\dot{\sqsubseteq}$ FishyMeatyPizza, Pizza $\dot{\sqsubseteq}$ FishyMeatyPizza, 
AnchoviesTopping $\dot{\sqsubseteq}$ FishTopping,\\ \hspace*{0.8cm}
PizzaTopping $\dot{\sqsubseteq}$ FishTopping,
ParmaHamTopping $\dot{\sqsubseteq}$ FishTopping,
MeatTopping $\dot{\sqsubseteq}$ FishTopping\}\\
\hline
{\bf ABox 1.2.3.1}: CLOSED\\
\hline
{\bf ABox 1.2.3.2}: \\
Pos$_x$ = \{MyPizza, Pizza\}, \\ Neg$_x$ = \{FishyMeatyPizza\} \\
Pos$_y$ = \{AnchoviesTopping, $\overline{AnchoviesTopping}$, PizzaTopping\}, \\ Neg$_y$ = \{FishTopping, $\overline{FishTopping}$\} \\
Pos$_z$ = \{ParmaHamTopping, $\overline{ParmaHamTopping}$, PizzaTopping, MeatTopping, $\overline{MeatTopping}$\}, \\ Neg$_z$ = \{FishTopping\}\\
$R_{\mathcal{A}}$ = \{MyPizza $\dot{\sqsubseteq}$ FishyMeatyPizza, Pizza $\dot{\sqsubseteq}$ FishyMeatyPizza, 
AnchoviesTopping $\dot{\sqsubseteq}$ FishTopping,\\ \hspace*{0.8cm}
PizzaTopping $\dot{\sqsubseteq}$ FishTopping,
ParmaHamTopping $\dot{\sqsubseteq}$ FishTopping,
MeatTopping $\dot{\sqsubseteq}$ FishTopping\} \\
\hline 
{\bf ABox 1.2.3.3}: \\
Pos$_x$ = \{MyPizza, Pizza\}, \\ Neg$_x$ = \{FishyMeatyPizza\} \\
Pos$_y$ = \{AnchoviesTopping, $\overline{AnchoviesTopping}$, PizzaTopping\}, \\ Neg$_y$ = \{FishTopping, $\overline{FishTopping}$\} \\
Pos$_z$ = \{ParmaHamTopping, $\overline{ParmaHamTopping}$, PizzaTopping\}, \\ Neg$_z$ = \{FishTopping, $\overline{FishTopping}$\}\\
$R_{\mathcal{A}}$ =\{MyPizza $\dot{\sqsubseteq}$ FishyMeatyPizza, Pizza $\dot{\sqsubseteq}$ FishyMeatyPizza, 
 AnchoviesTopping $\dot{\sqsubseteq}$ FishTopping,\\ \hspace*{0.8cm}
PizzaTopping $\dot{\sqsubseteq}$ FishTopping,
ParmaHamTopping $\dot{\sqsubseteq}$ FishTopping\} \\
\hline 
{\bf ABox 1.3.1}: CLOSED\\
\hline
{\bf ABox 1.3.2.1}: CLOSED\\
\hline
{\bf ABox 1.3.2.2}: \\
Pos$_x$ =  \{MyPizza, Pizza\}, \\ Neg$_x$ = \{FishyMeatyPizza\} \\
Pos$_y$ =  \{AnchoviesTopping, $\overline{AnchoviesTopping}$, PizzaTopping\}, \\ Neg$_y$ = \{MeatTopping,$\overline{MeatTopping}$\}\\
Pos$_z$ = \{ParmaHamTopping, $\overline{ParmaHamTopping}$, PizzaTopping\}, \\ Neg$_z$ = \{MeatTopping,$\overline{MeatTopping}$\}\\
$R_{\mathcal{A}}$ =\{MyPizza $\dot{\sqsubseteq}$ FishyMeatyPizza, Pizza $\dot{\sqsubseteq}$ FishyMeatyPizza, AnchoviesTopping $\dot{\sqsubseteq}$ MeatTopping, \\ \hspace*{0.8cm}
PizzaTopping $\dot{\sqsubseteq}$ MeatTopping,
ParmaHamTopping $\dot{\sqsubseteq}$ MeatTopping\}\\
\hline\end{tabular}
\end{center}
\caption{Creating $R_{\mathcal{A}}$ for the Leaf ABoxes related to MyPizza $\dot{\sqsubseteq}$ FishyMeaty\-Pizza.}
\label{rep-basic-algorithm}
\end{figure}

\begin{figure}[tb]
\begin{center}
\scriptsize
\begin{tabular}{| l |}
\hline
{\bf ABox 1}: \\
Pos$_x$ = \{MyPizza, Pizza\},  Neg$_x$ = \{FishyMeatyPizza\} \\
Pos$_y$ = \{AnchoviesTopping, $\overline{AnchoviesTopping}$, PizzaTopping\},  Neg$_y$ = $\emptyset$ \\
Pos$_z$ = \{ParmaHamTopping, $\overline{ParmaHamTopping}$, PizzaTopping\},  Neg$_z$ = $\emptyset$ \\
$R_{\mathcal{A}}$ = \{MyPizza $\dot{\sqsubseteq}$ FishyMeatyPizza, Pizza $\dot{\sqsubseteq}$ FishyMeatyPizza\} \\
\hline
{\bf ABox 1.1}: CLOSED\\
\hline
{\bf ABox 1.2}: \\
Pos$_x$ =  $\emptyset$,  Neg$_x$ = $\emptyset$ \\
Pos$_y$ = $\emptyset$,  Neg$_y$ = \{FishTopping\}\\
Pos$_z$ = $\emptyset$,  Neg$_z$ = \{FishTopping\}\\
$R_{\mathcal{A}}$ =\{AnchoviesTopping $\dot{\sqsubseteq}$ FishTopping,
PizzaTopping $\dot{\sqsubseteq}$ FishTopping,
ParmaHamTopping $\dot{\sqsubseteq}$ FishTopping\} \\
\hline 
{\bf ABox 1.2.1}: CLOSED\\
\hline
{\bf ABox 1.2.2}: \\
Pos$_x$ =  $\emptyset$, Neg$_x$ = $\emptyset$ \\
Pos$_y$ = \{MeatTopping, $\overline{MeatTopping}$\},  Neg$_y$ = $\emptyset$ \\
Pos$_z$ = $\emptyset$,  Neg$_z$ = $\emptyset$ \\
$R_{\mathcal{A}}$ =\{MeatTopping $\dot{\sqsubseteq}$ FishTopping\}\\
\hline 
{\bf ABox 1.2.2.1}: CLOSED\\
\hline
{\bf ABox 1.2.2.2}: \\
Pos$_x$ =  $\emptyset$,  Neg$_x$ = $\emptyset$ \\
Pos$_y$ = $\emptyset$,  Neg$_y$ = $\emptyset$ \\
Pos$_z$ = \{MeatTopping, $\overline{MeatTopping}$\},  Neg$_z$ = $\emptyset$ \\
$R_{\mathcal{A}}$ = $\emptyset$\\
\hline 
{\bf ABox 1.2.2.3}: \\
Pos$_x$ =  $\emptyset$,  Neg$_x$ = $\emptyset$ \\
Pos$_y$ = $\emptyset$,  Neg$_y$ = $\emptyset$ \\
Pos$_z$ = $\emptyset$,  Neg$_z$ = \{$\overline{FishTopping}$\}\\
$R_{\mathcal{A}}$ = $\emptyset$\\
\hline
{\bf ABox 1.2.3}: \\
Pos$_x$ =  $\emptyset$,  Neg$_x$ = $\emptyset$ \\
Pos$_y$ = $\emptyset$,  Neg$_y$ = \{$\overline{FishTopping}$\}\\
Pos$_z$ = $\emptyset$,  Neg$_z$ = $\emptyset$ \\
$R_{\mathcal{A}}$ = $\emptyset$\\
\hline 
{\bf ABox 1.2.3.1}: CLOSED\\
\hline
{\bf ABox 1.2.3.2}: \\
Pos$_x$ =  $\emptyset$, Neg$_x$ = $\emptyset$ \\
Pos$_y$ = $\emptyset$,  Neg$_y$ = $\emptyset$ \\
Pos$_z$ = \{MeatTopping, $\overline{MeatTopping}$\}, Neg$_z$ = $\emptyset$ \\
$R_{\mathcal{A}}$ =\{MeatTopping $\dot{\sqsubseteq}$ FishTopping\} \\
\hline 
{\bf ABox 1.2.3.3}: \\
Pos$_x$ =  $\emptyset$, Neg$_x$ = $\emptyset$ \\
Pos$_y$ = $\emptyset$,  Neg$_y$ = $\emptyset$ \\
Pos$_z$ = $\emptyset$, Neg$_z$ = \{$\overline{FishTopping}$\}\\
$R_{\mathcal{A}}$ = $\emptyset$\\
\hline
{\bf ABox 1.3}: \\
Pos$_x$ =  $\emptyset$, Neg$_x$ = $\emptyset$ \\
Pos$_y$ = $\emptyset$, Neg$_y$ = \{MeatTopping\}\\
Pos$_z$ = $\emptyset$, Neg$_z$ = \{MeatTopping\}\\
$R_{\mathcal{A}}$ =\{AnchoviesTopping $\dot{\sqsubseteq}$ MeatTopping,
PizzaTopping $\dot{\sqsubseteq}$ MeatTopping,
ParmaHamTopping $\dot{\sqsubseteq}$ MeatTopping\}\\
\hline 
{\bf ABox 1.3.1}: CLOSED\\
\hline
{\bf ABox 1.3.2}: \\
Pos$_x$ =  $\emptyset$, Neg$_x$ = $\emptyset$ \\
Pos$_y$ = $\emptyset$, Neg$_y$ = \{$\overline{MeatTopping}$\}\\
Pos$_z$ = $\emptyset$,  Neg$_z$ = $\emptyset$ \\
$R_{\mathcal{A}}$ = $\emptyset$\\
\hline 
{\bf ABox 1.3.2.1}: CLOSED\\
\hline
{\bf ABox 1.3.2.2}: \\
Pos$_x$ =  $\emptyset$, Neg$_x$ = $\emptyset$ \\
Pos$_y$ = $\emptyset$,  Neg$_y$ = $\emptyset$ \\
Pos$_z$ = $\emptyset$, Neg$_z$ = \{$\overline{MeatTopping}$\}\\
$R_{\mathcal{A}}$ = $\emptyset$\\
\hline\end{tabular}
\end{center}
\caption{Creating $R_{\mathcal{A}}$ for the ABoxes related to MyPizza $\dot{\sqsubseteq}$ FishyMeaty\-Pizza - optimized algorithm}
\label{rep-optimized-algorithm}
\end{figure}

\end{document}